\documentclass[10pt,twocolumn,letterpaper]{article}

\usepackage{cvpr}              

\usepackage{url}
\usepackage[utf8]{inputenc} %
\usepackage[T1]{fontenc}    %
\usepackage{xcolor}
\usepackage{textcomp}
\usepackage{float}
\usepackage{gensymb}
\usepackage{nicefrac}       %
\usepackage{microtype}      %
\usepackage{mathtools}
\usepackage{lipsum}
\usepackage{graphicx}
\usepackage[shortlabels]{enumitem}
\usepackage{wrapfig}
\usepackage{bm}
\usepackage{amsmath, amsthm, amsfonts}
\usepackage{amssymb}
\usepackage{thmtools, thm-restate}
\usepackage{algorithm}
\usepackage{algorithmicx}
\usepackage{multirow}
\usepackage{booktabs}

\usepackage{color}
\usepackage{mathrsfs}
\usepackage{colortbl}
\newtheorem*{remark}{Remark}
\usepackage{algpseudocode}

\DeclareMathOperator*{\argmax}{arg\,max}
\DeclareMathOperator*{\argmin}{arg\,min}
\newcommand{\sign}{\text{sign}}
\newtheorem{theorem}{Theorem}

\allowdisplaybreaks

\usepackage[pagebackref,breaklinks,colorlinks]{hyperref}
\usepackage{marvosym}

\usepackage[capitalize]{cleveref}
\crefname{section}{Sec.}{Secs.}
\Crefname{section}{Section}{Sections}
\Crefname{table}{Table}{Tables}
\crefname{table}{Tab.}{Tabs.}

\def\cvprPaperID{12401} 
\def\confName{CVPR}
\def\confYear{2023}
\makeatletter
\def\thanks#1{\protected@xdef\@thanks{\@thanks
        \protect\footnotetext{#1}}}
\makeatother

\begin{document}

\title{Robust Generalization against Photon-Limited Corruptions \\ via Worst-Case Sharpness Minimization}

\vspace{-5mm}
\author{
	Zhuo Huang$^{1,\dagger}$,\  
	Miaoxi Zhu$^{2,\dagger}$\thanks{\scriptsize{$^{\dagger}$Equal contribution. $^1$Syndey AI Centre, The University of Sydney; $^2$National Engineering Research Center for Multimedia Software, Institute of Artificial Intelligence, School of Computer Science and Hubei Key Laboratory of Multimedia and Network Communication Engineering, Wuhan University; $^3$JD Explore Academy; $^4$Department of Automation, University of Science and Technology of China; $^5$Key Laboratory of Intelligent Perception and Systems for High-Dimensional Information of Ministry of Education, School of Computer Science and Engineering, Nanjing University of Science and Technology; $^6$Department of Computer Science, Hong Kong Baptist University}. Correspondence to Jun Yu (\texttt{harryjun@ustc.edu.cn}). Our code is available at \href{https://github.com/zhuohuangai/SharpDRO}{https://github.com/zhuohuangai/SharpDRO}},\  
	Xiaobo Xia$^1$,\ 
	Li Shen$^3$,\ 
	Jun Yu$^{4, }\textsuperscript{\Letter}$,\\ 
	Chen Gong$^5$,\ 
	Bo Han$^6$,\ 
	Bo Du$^2$,\ 
	Tongliang Liu$^1$ \\[1ex]
}
\vspace{-3mm}

\maketitle

\vspace{-5mm}
%
\begin{abstract}
Robust generalization aims to tackle the most challenging data distributions which are rare in the training set and contain severe noises, i.e., photon-limited corruptions. Common solutions such as distributionally robust optimization (DRO) focus on the worst-case empirical risk to ensure low training error on the uncommon noisy distributions. However, due to the over-parameterized model being optimized on scarce worst-case data, DRO fails to produce a smooth loss landscape, thus struggling on generalizing well to the test set. Therefore, instead of focusing on the worst-case risk minimization, we propose SharpDRO by penalizing the sharpness of the worst-case distribution, which measures the loss changes around the neighbor of learning parameters. Through worst-case sharpness minimization, the proposed method successfully produces a flat loss curve on the corrupted distributions, thus achieving robust generalization. Moreover, by considering whether the distribution annotation is available, we apply SharpDRO to two problem settings and design a worst-case selection process for robust generalization. Theoretically, we show that SharpDRO has a great convergence guarantee. Experimentally, we simulate photon-limited corruptions using CIFAR10/100 and ImageNet30 datasets and show that SharpDRO exhibits a strong generalization ability against severe corruptions and exceeds well-known baseline methods with large performance gains.
\end{abstract}

\vspace{-8mm}
\section{Introduction}
\vspace{-1mm}
\label{sec:introduction}
Learning against corruptions has been a vital challenge in the practical deployment of computer vision models, as learning models are much more fragile to subtle noises than human perception systems~\cite{goodfellow2014explaining, hendrycks2016baseline, liu2021probabilistic}. During the training, the encountered corruptions are essentially perceived as distribution shift, which would significantly hinder the prediction results~\cite{liang2017enhancing, long2015learning, tzeng2017adversarial, xia2019anchor, xia2020part,xia2023moderate}. Therefore, to mitigate the performance degradation, enhancing generalization to corrupted data distributions has drawn lots of attention~\cite{arjovsky2019invariant,sagawa2019distributionally}.

\begin{figure}
	\centering
	\includegraphics[width=0.7\linewidth]{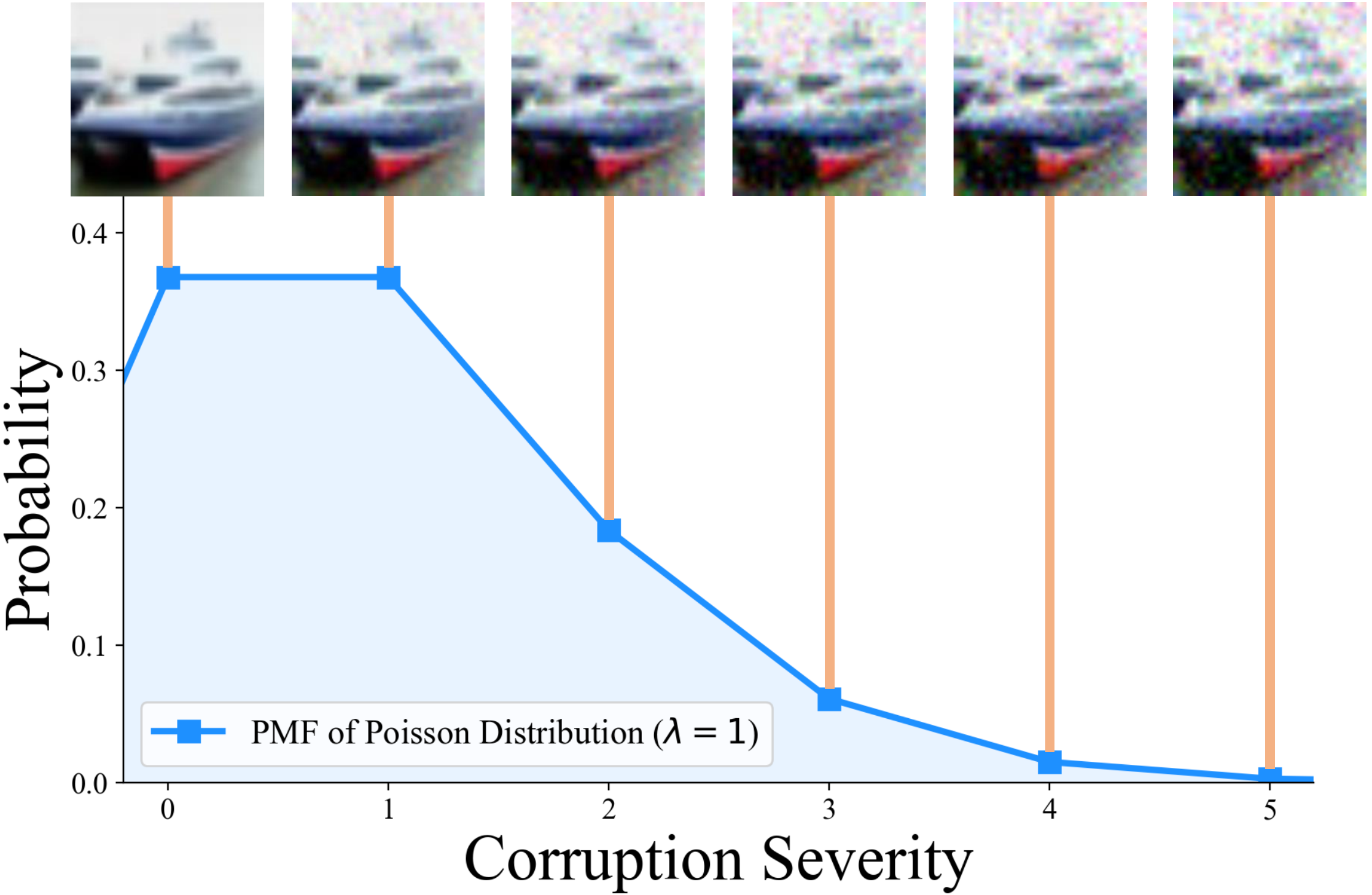}
	\vspace{-2mm}
	\caption{Illustration of photon-limited corruptions.}
	\vspace{-6mm}
	\label{fig:poisson}
\end{figure}

In the real world, noise corruptions are commonly known as photon-limited imaging problems~\cite{ingle2021passive, li2021photon, luisier2010image, timmermann1999multiscale} which arises due to numerous small corruption photons arriving to a image sensor. Consequently, different numbers of captured photons would form different levels of corruptions severity, further producing multiple data distributions and imposing varied impacts on learning models~\cite{hendrycks2019benchmarking}. Specifically, the encountered photon-limited corruption $\mathscr{E}$ is a composition of multiple noise photon $u$, which is triggered by some discrete factors with a certain probability during a time interval. For example, a photon $u$ can be triggered by each platform changing, re-distribution, transmission, etc. More photons are captured, and severer corruption would be applied to the image. Therefore, the severity $s$ of the photon-limited corruptions $\mathscr{E}$ can be modeled by Poisson distribution, \textit{i.e.}, $s\sim P\left( s; \lambda \right) = \frac{{e^{-\lambda } \lambda ^s }}{{s!}}$, which is illustrated in Figure~\ref{fig:poisson}. As a result, the real-world training set is not completely composed of clean data, but contains corrupted data with a smaller proportion as the severity goes stronger.

\begin{figure*}[t]
	\vspace{-10mm}
	\centering
	\includegraphics[width=0.9\linewidth]{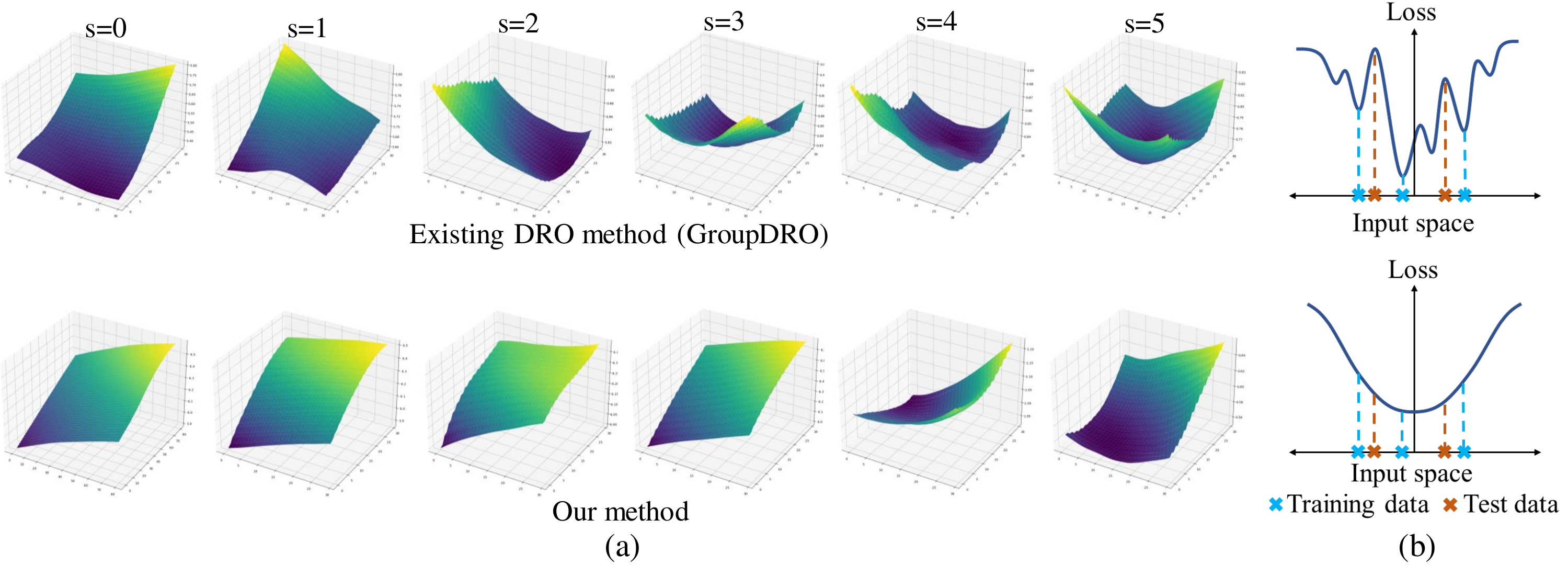}
	\vspace{-3mm}
	\caption{Illustration of our motivation. (a) Loss surface visualization of GroupDRO and the proposed SharpDRO. The columns from left to right stand for corrupted distributions with severity $s=0$ to $5$. (b) Illustration of why a sharp loss surface hinders generalization to test data.}
	\vspace{-5mm}
	\label{fig:visualization}
\end{figure*}

Dealing with such a realistic problem by vanilla empirical risk minimization can achieve satisfactory averaged accuracy on the whole training set. However, due to the extremely limited number of severely-corrupted data, the learning model would produce large training errors on the corrupted distributions, further hindering the robust performance under challenging real-world situations. A popular approach to achieve low error on the scarce corrupted data is distributionally robust optimization (DRO)~\cite{namkoong2016stochastic,sagawa2019distributionally,zhai2021doro, piratla2021focus, wang2022meta, wang2022improving}, which commonly optimizes the model parameter $\theta$ by optimizing:
\vspace{-1mm}
\begin{equation}
	\min_{\theta \in \Theta} \sup_{Q \in \mathcal{Q}} \mathbb{E}_{(x, y)\sim Q}\left[\mathcal{L}(\theta; (x, y))\right],
	\label{eq:general_dro}
	\vspace{-1mm}
\end{equation}
where $\mathcal{Q}$ denotes the uncertainty set that is utilized to estimate the possible test distribution. Intuitively, DRO assumes that $\mathcal{Q}$ consists of multiple sub-distributions, among which exists a worst-case distribution $Q$. By concentrating on the risk minimization of the worst-case distribution, DRO hopes to train a robust model which can deal with the potential distribution shift during the test phase. However, existing DRO methods usually leverage over-parameterized models to focus on a small portion of worst-case training data. Therefore, the worst-case data contaminated with severe corruption is highly possible to get stuck into sharp minima. As shown in the upper of Figure~\ref{fig:visualization} (a), a stronger corruption would cause the existing method to learn a sharper loss surface. Consequently, optimization via DRO fails to produce a flat loss landscape over the corrupted distributions, which leads to a large generalization gap between training and test set~\cite{keskar2017large,chaudhari2017entropy}.


To remedy this defect, in this paper, we propose SharpDRO method to focus on learning a flat loss landscape of the worst-case data, which can largely mitigate the training-test generalization gap problem of DRO. Specifically, we adopt the largest loss difference formed by applying weight perturbation~\cite{foret2020sharpness,wu2020adversarial} to measure the sharpness of the loss function. Intuitively, a sharp loss landscape is sensitive to noise and cannot generalize well on the test set. On the contrary, a flat loss landscape produces consistent loss values and is robust against perturbations (Figure~\ref{fig:visualization} (b)). By minimizing the sharpness, we can effectively enhance the generalization performance~\cite{keskar2017large, chaudhari2017entropy}. However, directly applying sharpness minimization on multiple distributions would yield poor results~\cite{cha2021swad}, as the computed sharpness could be influenced by the largest data distribution, and thus cannot generalize well to small corrupted data. Therefore, we only focus on worst-case sharpness minimization. In this way, as the lower of Figure~\ref{fig:visualization} (a) shows, SharpDRO successfully produces a flat loss surface, thus achieving robust generalization on the severely corrupted distributions.

In addition, identification of the worst-case distribution requires expensive annotations, which are not always practically feasible~\cite{liu2021just}. In this paper, we apply SharpDRO to solve two problem settings: 1) \textit{Distribution-aware robust generalization} which assumes that distribution indexes are accessible, and 2) \textit{Distribution-agnostic robust generalization} where the distributions are no longer identifiable, making the worst-case data hard to find. Existing approaches such as Just Train Twice (JTT) require two-stage training which is rather inconvenient. To tackle this challenge, we propose a simple (Out-of-distribution) OOD detection~\cite{hendrycks2016baseline,huang2021universal, huang2022they,liang2017enhancing, liu2020energy, wang2022watermarking, wang2023out} process to detect the worst-case data, which can be further leveraged to enable worst-case sharpness minimization. Through constructing training sets according to the Poisson distributed noisy distribution using CIFAR10/100 and ImageNet30, we show that SharpDRO can achieve robust generalization results on both two problem settings, surpassing well-known baseline methods by a large margin.

To sum up, our main contributions are three-fold:
\begin{itemize}
	\item We proposed a sharpness-based DRO method that overcomes the poor worst-case generalization performance of distributionally robust optimization.
	
	\item We apply the proposed SharpDRO method to both distribution-aware and distribution-agnostic settings, which brings a practical capability to our method. Moreover, we propose an OOD detection approach to select worst-case data to enable robust generalization.
	
	\item Theoretically, we show that SharpDRO has a convergence rate of $\mathcal{O}(\frac{\kappa^2}{\sqrt{MT}})$. Empirically, we form a photon-limited corruption dataset that follows Poisson distribution, and conduct extensive experiments to show a strong generalization ability of SharpDRO as well as its superiority to compared baseline methods.
\end{itemize}

In the following, we first briefly introduce the background and discuss the problem setting in section.~\ref{sec:background}. Then, we specify our SharpDRO over two problem settings in Section~\ref{sec:method}. Moreover, we give a detailed optimization process and provide convergence analysis in Section~\ref{sec:optimization}. Further, we conduct extensive experiments to validate our SharpDRO in Section~\ref{sec:experiments}. At last, we conclude this paper in Section~\ref{sec:conclusion}.
\section{Robust Generalization Methods}
\label{sec:background}
Due to the practical significance of robust generalization, various approaches have been proposed to deal with distribution shift. Here we briefly introduce three typical baseline methods, namely Invariant Risk Minimization, Risk Extrapolation, and GroupDRO.

\textbf{Invariant Risk Minimization (IRM)}~\cite{arjovsky2019invariant,chang2020invariant,creager2021environment} aims to extract the invariant feature across different distributions (also denoted as environments). Specifically, the learning model is separated into a feature extractor $G$ and a classifier $C$. IRM assumes an invariant model $C\circ G$ over various environments can be achieved if the classifier $C$ constantly stays optimal. Then, the learning objective is formulated as:
\vspace{-2mm}
\begin{equation}
	\small
	\begin{aligned}
		&\min_{C^*\circ G} \big\{\mathcal{L}_{\text{IRM}}:=\sum_{e\in \mathcal{E}}\mathcal{L}^e(C^*\circ G)\big\} \\
		&\text{s. t.}\  C^*\in\argmin_{G}\mathcal{L}^e(C\circ G), \text{for all}\  e \in \mathcal{E},
	\end{aligned}
	\label{eq:irm}
\vspace{-2mm}
\end{equation}
where $C^*$ stands for the optimal classifier, and $e$ denotes a specific environment from a given environmental set $\mathcal{E}$. By solving Eq.~\ref{eq:irm}, the feature extractor $G$ can successfully learn invariant information without being influenced by the distribution shift between different environments.

\textbf{Risk Extrapolation (REx)}~\cite{krueger2021out} targets at generalization to out-of-distribution (OOD) environments. Inspired by the discovery that penalizing the loss variance across distributions helps achieve improved performance on OOD generalization, REx proposes to optimizing via:
\vspace{-2mm}
\begin{equation}
	\small
	\min_{\theta\in\Theta}\big\{ \mathcal{L}_{\text{REx}}:=\sum_{e\in \mathcal{E}}\mathcal{L}^e(\theta) + \beta Var(\mathcal{L}^e,..., \mathcal{L}^m)\big\},
	\label{eq:rex}
\vspace{-2mm}
\end{equation}
where $\beta$ controls the penalization level. Intuitively, REx seeks to achieve risk fairness among all $m$ training environments, so as to increase the similarity of different learning tasks. As a result, the training model can capture the invariant information that helps generalize to unseen distributions.

\textbf{GroupDRO}~\cite{sagawa2019distributionally, hashimoto2018fairness, piratla2021focus} deal with the situation when the correlation between class label $y$ and unknown attribute $a$ differs in the training and test set. Such a difference is called spurious correlation which could seriously misguide the model prediction. As a solution, GroupDRO considers each combination of class and attribute as a group $g$. By conducting risk minimization through:
\vspace{-2mm}
\begin{equation}
	\small
	\min_{\theta \in \Theta} \big\{\mathcal{L}_{\text{GroupDRO}}:=\max_{g} \mathbb{E}_{(x, y)\sim P_g}\left[\mathcal{L}(\theta; (x, y))\right]\big\},
	\label{eq:group_dro}
\vspace{-2mm}
\end{equation}
the worst-case group from distribution $P_g$ which commonly holds spurious correlation is emphasized, thus breaking the spurious correlation.

\noindent
\textbf{Discussion:} IRM and REx both focus on learning invariant knowledge across various environments. However, when the training set contains extremely imbalanced noisy distributions, as shown in Figure~\ref{fig:poisson}, the invariant learning results would be greatly misled by the most dominating distribution. Thus, the extracted invariant feature would be questionable for generalization against distribution shift. Although emphasizing the risk minimization of worst-case data via GroupDRO can alleviate the imbalance problem, its generalization performance is still sub-optimal when facing novel test data. However, SharpDRO can not only focus on the uncommon corrupted data but also effectively improve the generalization performance on the test set by leveraging worst-case sharpness minimization.

Our investigated problem is closely related to OOD generalization which is a broad field that contains many popular research topics, such as \textbf{Domain Generalization}~\cite{carlucci2019domain, peng2019moment, qiao2020learning, shu2021open, mahajan2021domain, muandet2013domain, huang2023harnessing, zhang2022towards}, \textbf{Causal Invariant Learning}~\cite{arjovsky2019invariant, krueger2021out, li2018deep, yang2021causalvae, scholkopf2021toward, yue2021transporting}. Generally, existing works mainly studies two types of research problem: 1) mitigating domain shift between training and test dataset; and 2) breaking the spurious correlation between causal factors. However, as generalization against corruptions \textbf{does not introduce any domain shift or spurious correlation}, such a problem cannot be naively solved by domain generalization methods or causal representation learning techniques. Therefore, in this paper, we focus on complementing this rarely-explored field and propose SharpDRO to enforce robust generalization against corruption. In the next section, we elaborate on the methodology of SharpDRO.
\begin{figure*}
	\vspace{-8mm}
	\begin{minipage}[t]{0.32\textwidth}
		\centering
		\includegraphics[width=\linewidth]{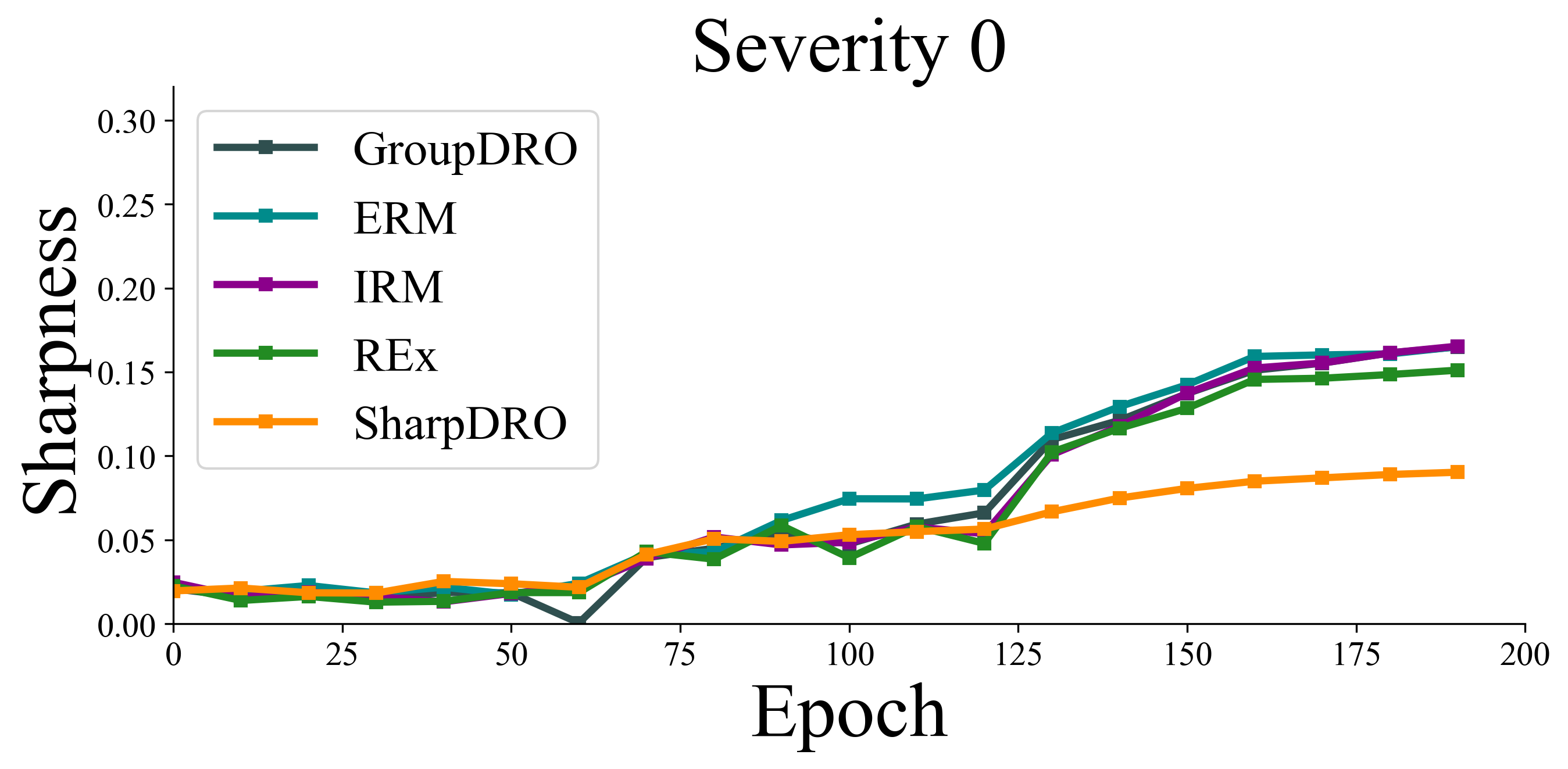}
	\end{minipage}
	\begin{minipage}[t]{0.32\textwidth}
		\centering
		\includegraphics[width=\linewidth]{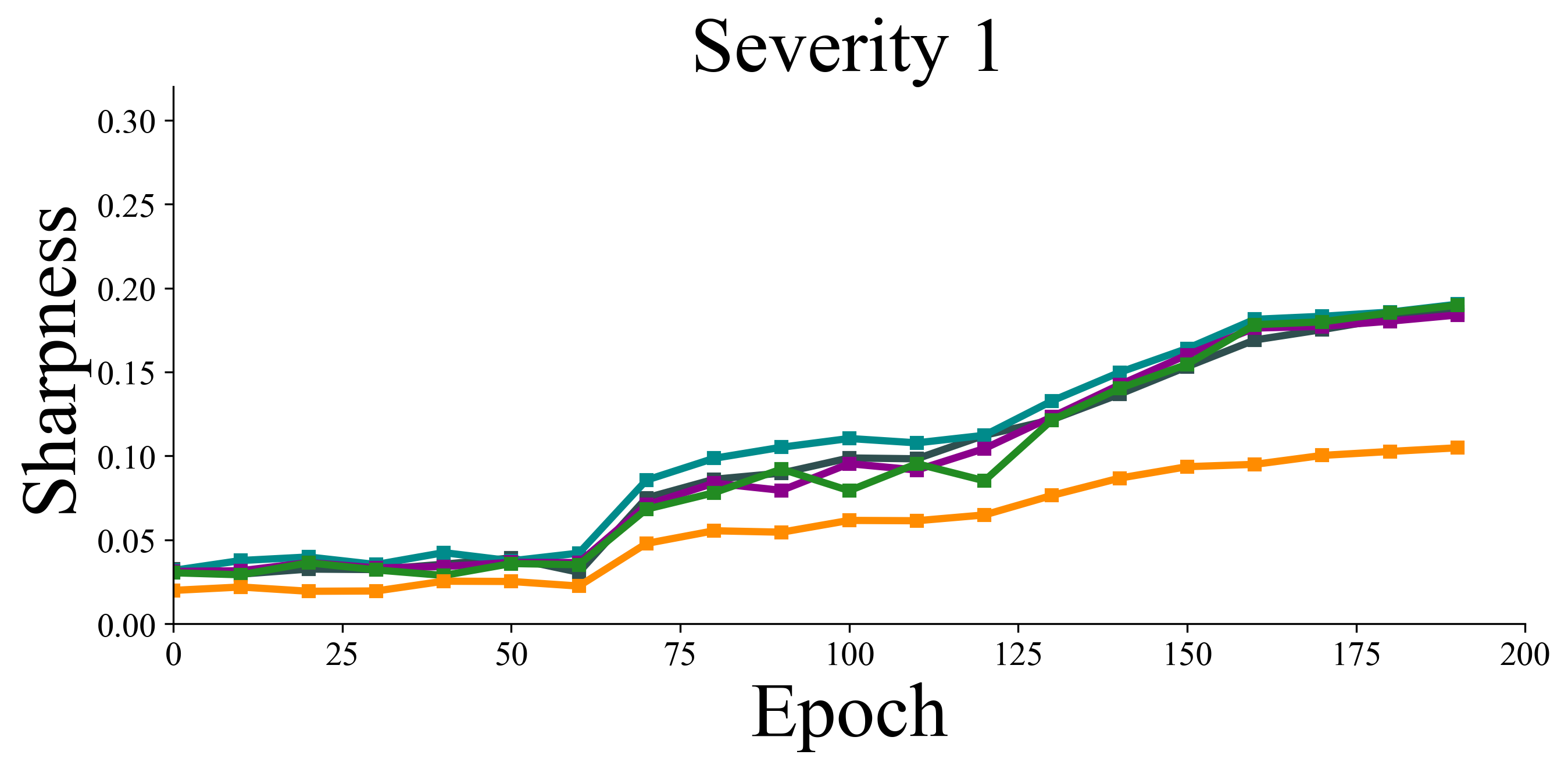}
	\end{minipage}
	\begin{minipage}[t]{0.32\textwidth}
		\centering
		\includegraphics[width=\linewidth]{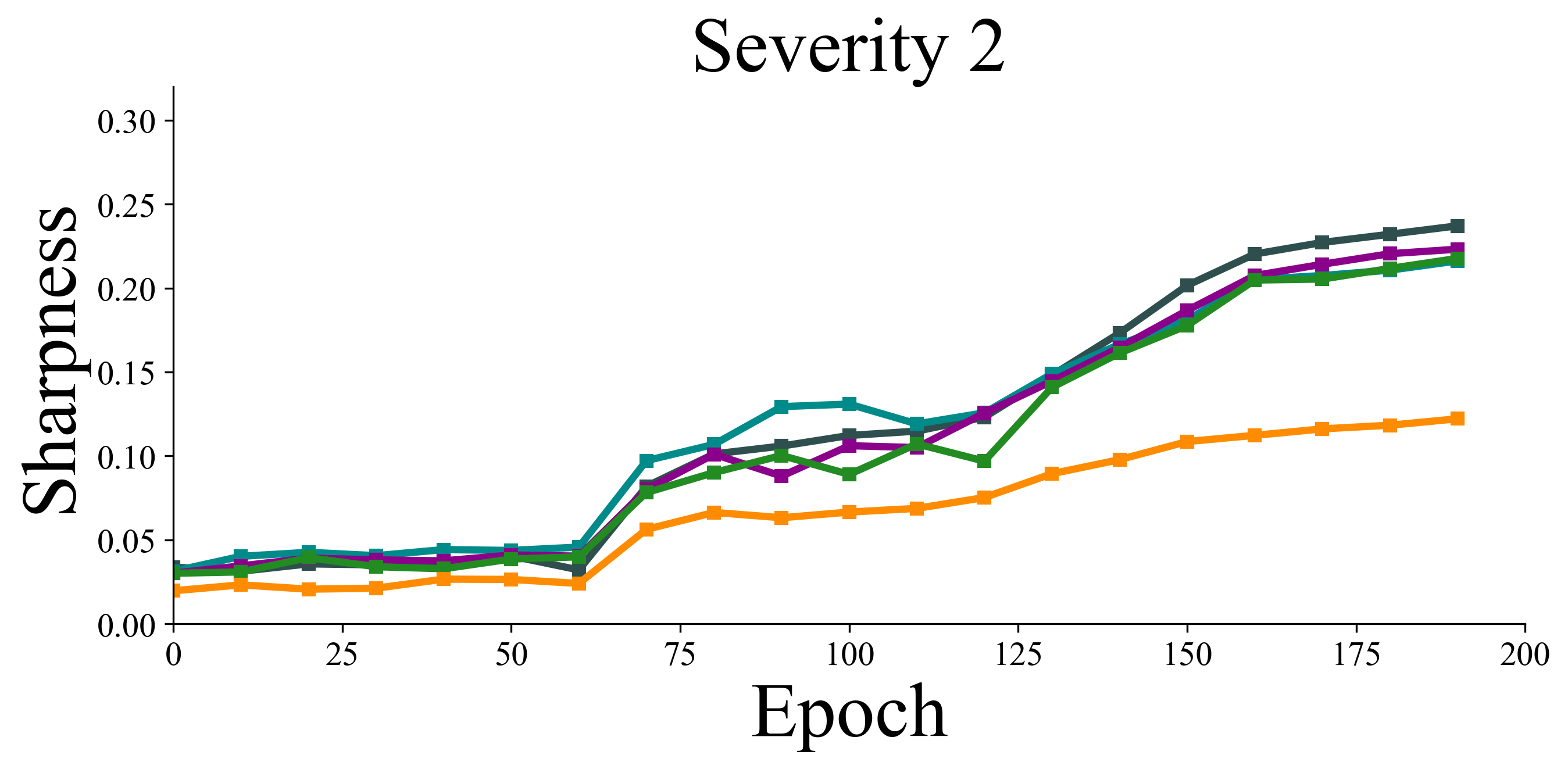}
	\end{minipage}\\
	\begin{minipage}[t]{0.32\textwidth}
		\centering
		\includegraphics[width=\linewidth]{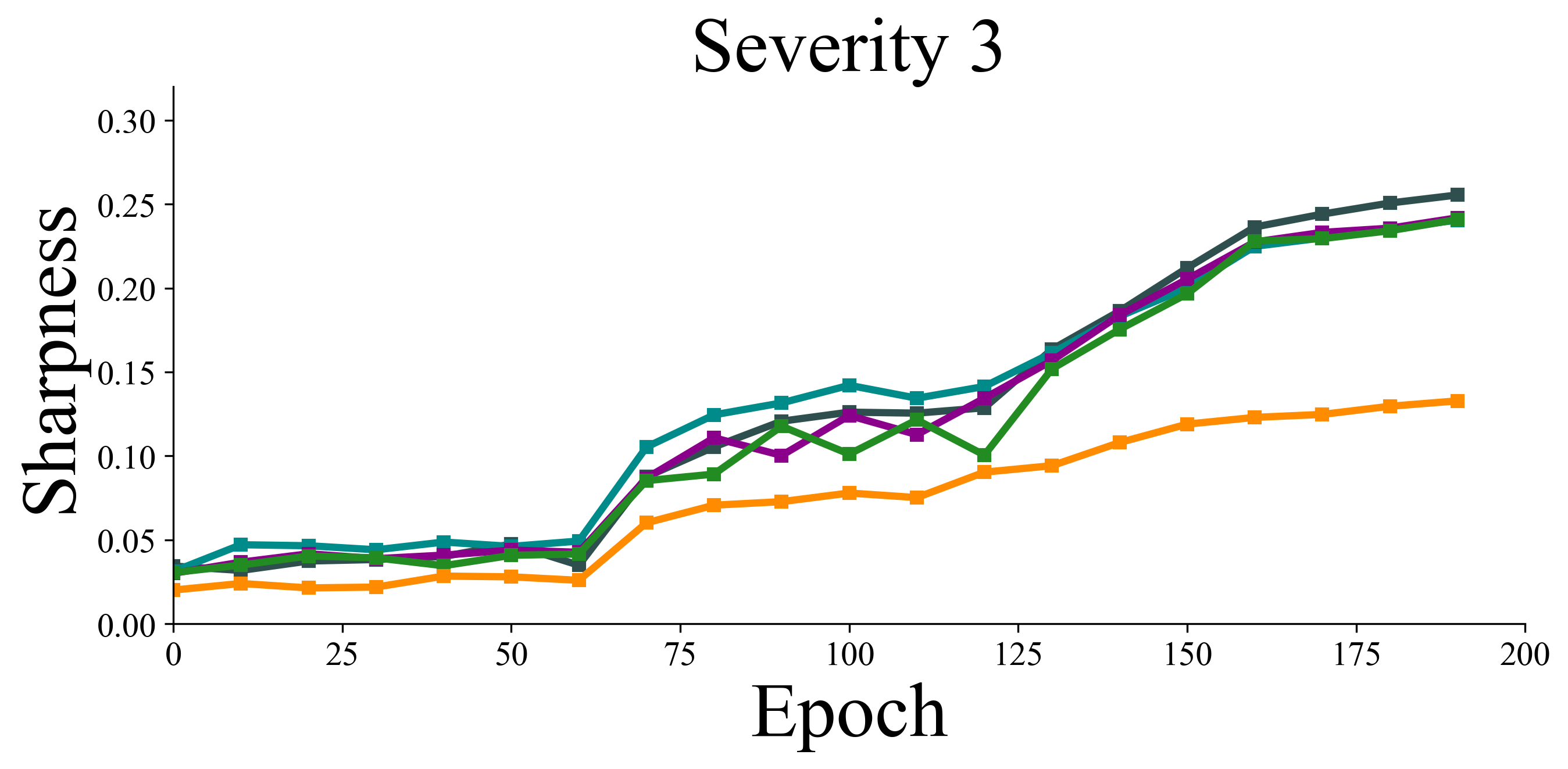}
	\end{minipage}
	\begin{minipage}[t]{0.32\textwidth}
		\centering
		\includegraphics[width=\linewidth]{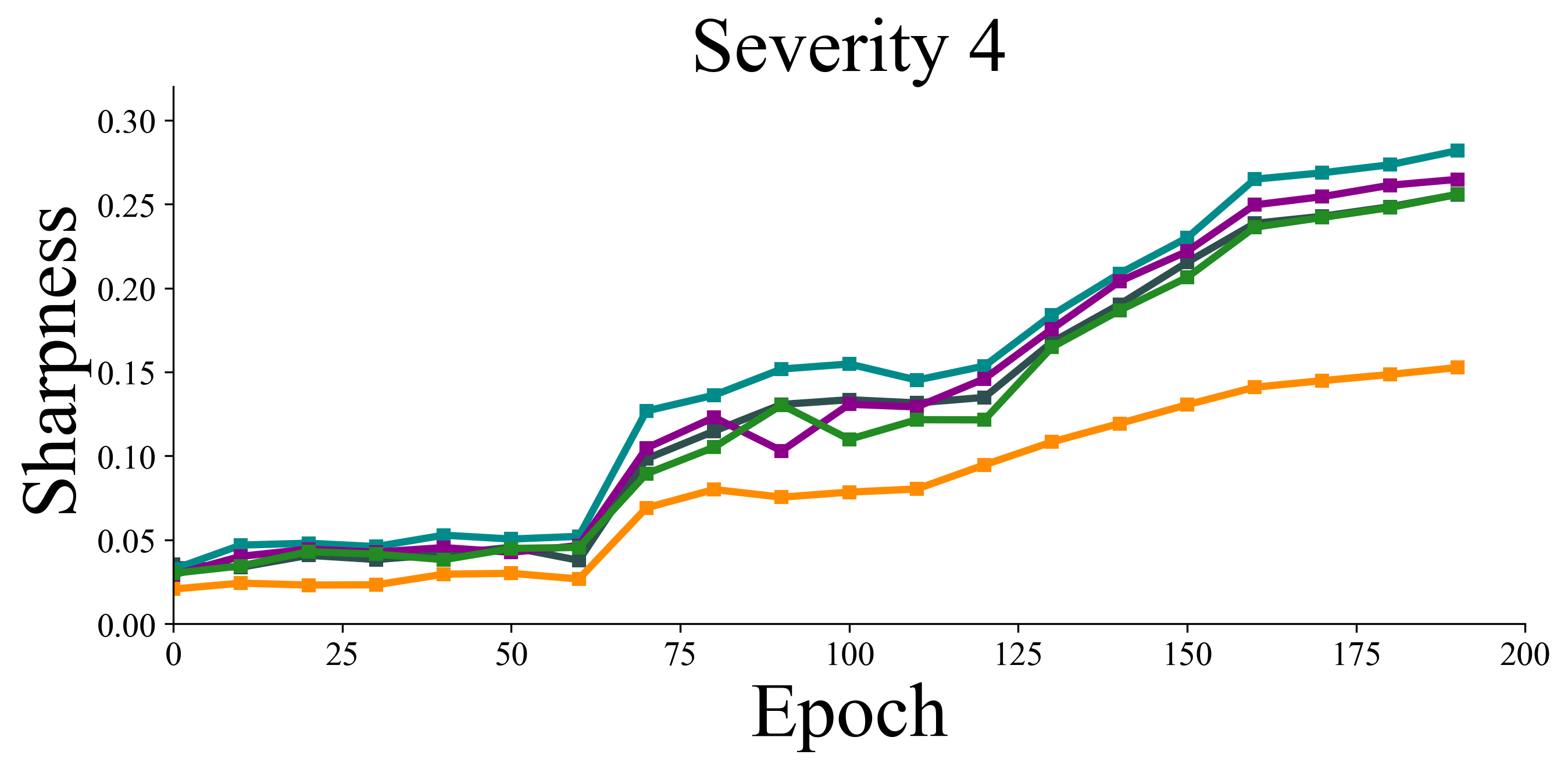}
	\end{minipage}
	\begin{minipage}[t]{0.32\textwidth}
		\centering
		\includegraphics[width=\linewidth]{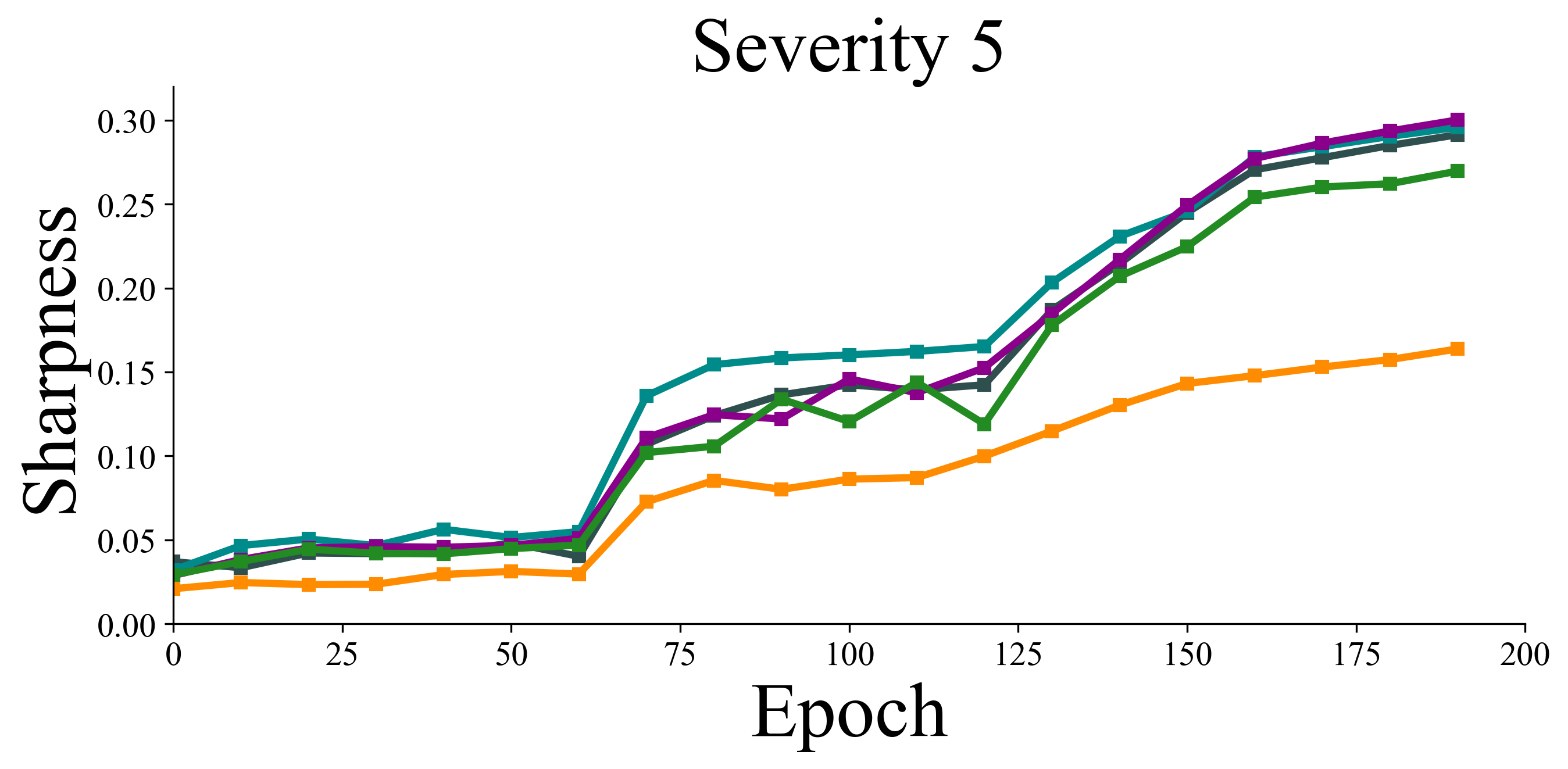}
	\end{minipage}
	\vspace{-4mm}
	\caption{\small Sharpness during networking training on clean ($s=0$) and corrupted distributions ($s=1$ to $5$).}
	\vspace{-5mm}
	\label{fig:sharpness}
\end{figure*}

\section{Methodology}
\label{sec:method}

In robust generalization problems, we are given a training set $\mathcal{D}^{\text{train}}$ containing $n$ image examples, each example $x\in\mathcal{X}$ is given a class label $y\in\mathcal{Y}=\left\{1, 2, ..., c\right\}$. Moreover, the training set is corrupted by a certain type of noise whose severity $s$ follows a Poisson distribution $P(s; \lambda)$. Here we assume $\lambda=1$ which indicates that the mean number of the noise photon $u$ that occurred during a time interval is $1$. Therefore, the distribution $P$ of the whole training set is composed of $S$ sub-distributions $P_s, s\in\left\{1, 2, ..., S\right\}$ with varied levels of corruption. Our goal is to learn a robust model $\theta \in \Theta$ that can achieve good generalization performance on challenging data distributions $P_s$ with large severity.

The general objective of SharpDRO is formulated as:
\vspace{-1mm}
\begin{equation}
	\begin{aligned}
		\min_{\theta} \left\{\mathcal{L}_{\text{SharpDRO}}:= \mathbb{E}_{(x, y)\sim Q}\left[\mathcal{L}(\theta; (x, y))\right] +\right.\\
		\left. \mathbb{E}_{(x, y)\sim Q}\left[\mathcal{R}(\theta; (x, y))\right]\right\},
	\end{aligned}
	\label{eq:SharpDRO}
	\vspace{-1mm}
\end{equation}
where the first term denotes the risk minimization using loss function $\mathcal{L}$, meanwhile a worst-case distribution $Q$ is selected based on the model prediction. The second term $\mathcal{R}$ indicates the sharpness minimization which aims to maximally improve the generalization performance on the worst-case distribution $Q$. Specifically, as shown in Figure~\ref{fig:sharpness}, the sharpness gradually increases as the corruption severity enlarges. Therefore, to accomplish robust generalization, we are motivated to emphasize the worst-case distribution. As a result, we can produce much smaller sharpness compared to other methods, especially on severely corrupted distributions.

In the following, we first introduce worst-case sharpness for robust generalization. Then, we demonstrate worst-case data selection on two scenarios. Finally, we provide a detailed optimization process and convergence analysis.

\subsection{Sharpness for Robust Generalization}

The main challenge of robust generalization is that the training distribution is extremely imbalanced, as shown in Figure~\ref{fig:poisson}. The performance on the abundant clean data is quite satisfactory, but robustness regarding the corrupted distribution is highly limited, owing to the severe disturbance of corruption as well as the insufficiency of noisy data. To enhance the generalization performance, we leverage sharpness to fully exploit the worst-case data. Specifically, sharpness~\cite{foret2020sharpness, kim2022fisher, liu2022towards,wu2020adversarial,zheng2021regularizing} is measured by \textbf{the largest loss change} when model weight $\theta$ is perturbed with $\epsilon$, formally:
\vspace{-0.5mm}
\begin{equation}
    \small
	\mathcal{R}:=\max_{\|\epsilon\|_2\le\rho}\{\mathcal{L}(\theta+\epsilon; (x, y)) - \mathcal{L}(\theta; (x, y))\},
	\label{eq:sharpness}
\vspace{-0.5mm}
\end{equation}
where $\rho$ is a scale parameter to control the perturbation magnitude. By supposing $\epsilon$ is small enough, we can have:
\vspace{-0.5mm}
\begin{equation}
\small
	\mathcal{L}(\theta+\epsilon) - \mathcal{L}(\theta)\approx\nabla\mathcal{L}(\theta)\epsilon.
\vspace{-0.5mm}
\end{equation}
Further, we hope to obtain the largest loss change to find the optimal weight perturbation $\epsilon^*$, which is be computed as:
\vspace{-0.5mm}
\begin{equation}
\small
	\epsilon^*:=\argmax_{\|\epsilon\|_2\le\rho}{\nabla\mathcal{L}(\theta)\epsilon}.
\vspace{-0.5mm}
\end{equation}
By following dual norm problem, the optimal $\epsilon^*$ can be solved as $\rho\sign(\nabla\mathcal{L}(\theta))$~\cite{foret2020sharpness}, which is essentially the $\infty$-norm of the gradient $\nabla\mathcal{L}$ multiplied with a scale parameter $\rho$. Hence, common sharpness minimization aims to minimize:
\begin{equation}
\small
	\mathbb{E}_{(x, y)\sim Q}\mathcal{R}:=\mathcal{L}(\theta+\rho\sign(\nabla\mathcal{L}(\theta; (x, y)))) - \mathcal{L}(\theta; (x, y)).
	\label{eq:sharpness_minimization}
\end{equation}
The intuition is that the perturbation along the gradient norm direction increases the loss value significantly. When training on corrupted distributions, the scarce noisy data scatter sparsely in the high-dimensional space. As a consequence, the neighbor of each datum could not be sufficiently explored, thus producing a sharp loss curve. During test, the unseen noisy data is likely to fall on an unexplored point with a large loss, further causing inaccurate model predictions.

Therefore, instead of directly applying sharpness minimization on the whole dataset, which leads to poor generalization performance~\cite{cha2021swad} (as demonstrated in Section~\ref{sec:ablation_study}), we focus on sharpness minimization over the worst-case distribution $Q$. By conducting the worst-case sharpness minimization, we can enhance the flatness of our classifier. Consequently, when predicting unknown data during the test phase, a flat loss landscape is more likely to produce a low loss than a sharp one, hence our SharpDRO can generalize better than other DRO methods. However, the robust performance largely depends on the worst-case distribution $Q$, so next, we explain our worst-case data selection.

\subsection{Worst-Case Data Selection}
Generally, the worst-case data selection focuses on finding the most uncertain data distribution $Q$ from the uncertainty set $\mathcal{Q}$, which is a $f$-divergence ball from the training distribution $P$~\cite{ben2013robust, duchi2018learning, hu2018does}. Most works assume each distribution is distinguishable from the other. However, when the distribution index is not available, it would be very hard to select worst-case data. In this section, we investigate two situations: distribution-aware robust generalization and distribution-agnostic robust generalization.

\subsubsection{Distribution-Aware Robust Generalization}
When given annotations to denote different severity of corruptions, the image data $x$ is paired with class label $y$ and distribution index $s$. Then, the worst-case distribution $Q$ can be found by identifying the sub-distribution $P_s\in P$ that yields the largest loss. Hence, we can optimize through:
\vspace{-1mm}
\begin{align}
\small
	\min_{\theta} & \Big\{\max_{\omega_s; \atop Q = \{\omega_s P_s\}_{s=1}^S} \big\{\sum_{(x_i, y_i)\in Q}\left[\mathcal{L}(\theta,\omega_s; (x_i, y_i))\right]\big\}  \label{eq:distribution_aware}\\
	&\qquad + \sum_{(x_i, y_i)\in Q}\left[\mathcal{R}(\theta,\omega_s; (x_i, y_i))\right]\Big\}, \nonumber
\vspace{-3mm}
\end{align}
where $\omega_s$ belongs to a $(S-1)$-dimensional probability simplex. The first term simply recovers the learning target of GroupDRO~\cite{sagawa2019distributionally,hu2018does} and helps find the worst-case distribution $Q$. Then, by emphasizing the selected $P_s$, the second sharpness minimization term can act as a sharpness regularizer. As a result, SharpDRO can learn a flatter loss surface on the worst-case data, thus generalize better compared to GroupDRO, as discussed in Section~\ref{sec:experiments}.

\subsubsection{Distribution-Agnostic Robust Generalization}
Due to the annotations being extremely expensive in the real world, a practical challenge is how to learn a robust model without a distribution index. Unlike JTT~\cite{liu2021just} which trains the model through two stages, we aim to solve this problem more efficiently by detecting the worst-case data during network training. As the corrupted data essentially lie out-of-distribution, so we are motivated to conduct OOD detection~\cite{hendrycks2019deep, li2022out, liang2017enhancing, liu2020energy, wei2022mitigating} to find the worst-case data.

Particularly, we re-utilize the previously computed weight perturbation $\epsilon^*$ to compute an OOD score:
\begin{equation}
	\omega_i = \max f(\theta; (x_i)) - \max f(\theta+\epsilon^*; (x_i)),
	\label{eq:ood_score}
\end{equation}
where $f(\cdot)$ stands for the $c$-dimensional label prediction in the label space, whose maximum value is considered as prediction confidence. Intuitively, as the model is much more robust to the clean distribution than the corrupted distribution, the prediction of clean data usually exhibits more stability than scarce noisy data when facing perturbations. Hence, if an example comes from a rarely explored distribution, its prediction certainty would deviate significantly from the original value, thus producing a large OOD score, as shown in Section~\ref{sec:ood_score}. Note that the major difference is that we target generalization on worst-case data, but OOD detection aims to exclude OOD data.

To this end, we can construct our worst-case dataset as $Q:=\left\{\sum_{i=1}^M \bar{\omega}_i\cdot(x_i, y_i): \bar{\omega}_i=\frac{\omega_i}{\frac{1}{M}\sum_{i=1}^M \omega_i}\right\}$, where normalization on $\omega_i$ is performed simultaneously. Then, the learning target of the distribution-agnostic setting becomes:
\begin{align}
	\small
		\min_{\theta}& \Big\{\max_{\bar{\omega}_i} \big\{\sum_{(x_i, y_i)\in Q}\left[\mathcal{L}(\theta,\bar{\omega}_i; (x, y))\right]\big\} \label{eq:distribution_agnostic}\\
		&\qquad +  \sum_{(x_i, y_i)\in Q}\left[\mathcal{R}(\theta,\bar{\omega}_i; (x, y))\right]\Big\}, \nonumber
\end{align}
Therefore, the worst-case data can be selected by focusing on the examples with large OOD scores. In this way, our sharpDRO can be successfully deployed into the distribution-agnostic setting to ensure robust generalization, whose effectiveness is demonstrated by quantitative and qualitative results in Sections~\ref{sec:distribution_agnostic} and~\ref{sec:ood_score}. Next, we give details about implementing SharpDRO.


\renewcommand{\algorithmicrequire}{\textbf{Input:}}
\renewcommand{\algorithmicensure}{\textbf{Output:}}
\begin{algorithm}[H]
	\small
	\caption{\small Optimization process of SharpDRO}
	\label{alg:sharpdro}
	\begin{algorithmic}[1]
		\State Training set $\mathcal{D}^{\text{train}}=\left\{x_i, y_i\right\}_{i=1}^M$ containing Poisson distributed noisy corruptions; Model parameter $\theta \gets \theta_0$; Weighting parameter $\omega \gets \omega_0$; Learning rate: $\eta_{\theta}$, $\eta_{\omega}$.
		\For{$t \in 0,1, \ldots,T-1$}
		\If {\textit{Distribution-aware}}
		\State \textit{\color{black!60} $\triangleright$ Loss maximization via optimizing $\omega_{t+1}$}
		\State $\omega_{t+1} := \arg\max_{\omega} \left\{\mathbb{E}_{(x, y)\sim \omega P_s}\left[\mathcal{L}(\theta_{t},\omega; (x, y))\right]\right\}$; 
		\ElsIf {\textit{Distribution-agnostic}}
		\State{\textit{\color{black!60} $\triangleright$ OOD detection for computing $\omega_{t+1}$}}
		\State Update $\omega_{t+1}$ via Eq.~\ref{eq:ood_score}; 
		\EndIf
		\State \textit{\color{black!60} $\triangleright$ Optimize variable $\theta$}
		\State $\theta_{t+1} = \arg\min_{\theta}\Big\{\mathbb{E}_{(x,y)\sim w_tP}[\mathcal{L}(\theta,\omega_t) + \mathcal{R}(\theta,\omega_t)]  \Big\}$
		\EndFor
	\end{algorithmic}
\end{algorithm}

\vspace{-3mm}
\subsection{Optimization for SharpDRO}
\label{sec:optimization}
In both distribution-aware and distribution-agnostic scenarios, the worst-case data distribution is identified using distribution weighting parameter $\omega_s$ and OOD score $\omega$, respectively. Intuitively, their effect is similar: finding the worst data distribution which yields the maximum loss. Therefore, without loss of generality, we consider the maximization in Eqs.~\ref{eq:distribution_aware} and~\ref{eq:distribution_agnostic} as the optimization on the same weighting parameter $\omega$ and the samples $(x,y)$ can be considered i.i.d. from $\mathcal{Q}$ weighted by $\omega$ to compute the loss value $\mathcal{L}$. Moreover, the sharpness regularization can be reformulated in the same way as Eq.~\ref{eq:sharpness} by including $\omega$: $\mathcal{R}(\theta,\omega;(x,y))=\max_{\|\epsilon\|_2\leq \rho}\{\mathcal{L}(\theta+\epsilon,\omega;(x,y))-\mathcal{L}(\theta,\omega;(x,y))\}$. Therefore, our general learning objective can be formulated as a bi-level optimization problem:
\vspace{-2mm}
\begin{align}
\small
& \min_{\theta} \mathbb{E}_{(x, y)\sim Q}\left[\mathcal{L}(\theta,\omega^{*}; (x, y))\right] + \mathcal{R}(\theta,\omega^{*}; (x, y))   \\
& \qquad \text{s. t.}\ \omega^{*} = \arg\max_{\omega} \mathbb{E}_{(x, y)\sim Q}\left[\mathcal{L}(\theta,\omega; (x, y))\right].
\vspace{-4mm}
\end{align}
The optimization process is shown in Algorithm.~\ref{alg:sharpdro}. Specifically, we first update the weighting parameter $\omega$ based on the empirical risk term $\mathcal{L}$ using stochastic gradient ascent. Then, by leveraging the updated $\omega$, we optimize the general objective which contains both risk minimization of $\mathcal{L}$ and worst-case sharpness minimization of $\mathcal{R}$. We iterate these processes until convergence, hoping to minimize the risk on target loss function $\mathcal{L}$ with the worst-case data distribution.

\noindent
\textbf{Convergence Analysis:}
	First we give some brief notations: $\mathbb{L}(\theta,\omega):=\mathbb{E}_{(x,y)\sim Q}\mathcal{L}(\theta,\omega;(x,y))$. The worst-case data distribution which has the maximum loss is denoted by $\omega^*(\theta):= \arg\max_{\omega}\mathbb{L}(\theta,\omega)$. We can obtain the convergence to a stationary point of $\mathbb{L}^*(\theta):=\max_{\omega}\mathbb{L}(\theta,\omega)=\mathbb{L}(\theta,\omega^*(\theta))$ by averaged gradient 
	 $\frac{1}{T}\sum_{t=0}^{T-1}\mathbb{E}\|\nabla \mathbb{L}^*(\theta_t)\|^2$.

\begin{table*}
\vspace{-8mm}
	\scriptsize
	\centering
	\caption{Quantitative comparisons on distribution-aware robust generalization setting. Averaged accuracy ($\%$) with standard deviations are computed over three independent trials.}
	\vspace{-0.2cm}
	\setlength{\tabcolsep}{2.8mm}
	\label{tab:distribution_aware}
	\begin{tabular}{lllcccccc}
		\toprule[1pt]
		\multirow{2}{*}{data} & \multirow{2}{*}{Type} & \multirow{2}{*}{Method} & \multicolumn{6}{c}{Corruption Severity} \\
		&  &  & \multicolumn{1}{c}{0} & \multicolumn{1}{c}{1} & \multicolumn{1}{c}{2} & \multicolumn{1}{c}{3} & \multicolumn{1}{c}{4} & \multicolumn{1}{c}{5} \\ \midrule[0.6pt]
		\multirow{10}{*}{CIFAR10} & \multirow{5}{*}{Gaussian} & ERM & \multicolumn{1}{c}{$90.9\pm0.02$} & \multicolumn{1}{c}{$89.2\pm0.02$} & \multicolumn{1}{c}{$86.4\pm0.03$} & \multicolumn{1}{c}{$85.9\pm0.01$} & \multicolumn{1}{c}{$83.5\pm0.01$} & \multicolumn{1}{c}{$78.8\pm0.01$} \\
		&  & IRM & \multicolumn{1}{c}{$91.8\pm0.01$} & \multicolumn{1}{c}{$90.3\pm0.01$} & \multicolumn{1}{c}{$89.5\pm0.01$} & \multicolumn{1}{c}{$86.7\pm0.02$} & \multicolumn{1}{c}{$81.8\pm0.02$} & \multicolumn{1}{c}{$80.0\pm0.02$} \\
		&  & REx & \multicolumn{1}{c}{$91.3\pm0.03$} & \multicolumn{1}{c}{$89.5\pm0.02$} & \multicolumn{1}{c}{$88.1\pm0.02$} & \multicolumn{1}{c}{$86.7\pm0.02$} & \multicolumn{1}{c}{$83.3\pm0.01$} & \multicolumn{1}{c}{$80.5\pm0.02$} \\
		&  & GroupDRO & \multicolumn{1}{c}{$90.2\pm0.03$} & \multicolumn{1}{c}{$89.1\pm0.02$} & \multicolumn{1}{c}{$88.4\pm0.04$} & \multicolumn{1}{c}{$84.3\pm0.01$} & \multicolumn{1}{c}{$83.0\pm0.02$} & \multicolumn{1}{c}{$78.2\pm0.02$} \\
		&  & SharpDRO & \multicolumn{1}{c}{$\bm{92.9}\pm\bm{0.02}$} & \multicolumn{1}{c}{$\bm{91.3}\pm\bm{0.02}$} & \multicolumn{1}{c}{$\bm{90.5}\pm\bm{0.01}$} & \multicolumn{1}{c}{$\bm{88.4}\pm\bm{0.02}$} & \multicolumn{1}{c}{$\bm{86.9}\pm\bm{0.01}$} & \multicolumn{1}{c}{$\bm{84.7}\pm\bm{0.01}$} \\ \cline{2-9} 
		
		& \multirow{5}{*}{JPEG} & ERM & \multicolumn{1}{c}{$91.0\pm0.02$} & \multicolumn{1}{c}{$89.0\pm0.02$} & \multicolumn{1}{c}{$86.2\pm0.02$} & \multicolumn{1}{c}{$83.1\pm0.02$} & \multicolumn{1}{c}{$82.5\pm0.03$} & \multicolumn{1}{c}{$81.4\pm0.03$} \\
		&  & IRM & \multicolumn{1}{c}{$90.2\pm0.02$} & \multicolumn{1}{c}{$88.2\pm0.02$} & \multicolumn{1}{c}{$86.7\pm0.03$} & \multicolumn{1}{c}{$84.0\pm0.02$} & \multicolumn{1}{c}{$82.9\pm0.03$} & \multicolumn{1}{c}{$81.6\pm0.02$}  \\
		&  & REx & \multicolumn{1}{c}{$89.6\pm0.03$} & \multicolumn{1}{c}{$89.2\pm0.02$} & \multicolumn{1}{c}{$86.0\pm0.03$} & \multicolumn{1}{c}{$85.8\pm0.03$} & \multicolumn{1}{c}{$82.7\pm0.03$} & \multicolumn{1}{c}{$81.9\pm0.02$}  \\
		&  & GroupDRO & \multicolumn{1}{c}{$90.3\pm0.02$} & \multicolumn{1}{c}{$88.6\pm0.03$} & \multicolumn{1}{c}{$86.5\pm0.03$} & \multicolumn{1}{c}{$84.2\pm0.02$} & \multicolumn{1}{c}{$83.2\pm0.02$} & \multicolumn{1}{c}{$82.1\pm0.02$} \\
		&  & SharpDRO & \multicolumn{1}{c}{$\bm{91.2}\pm\bm{0.01}$} & \multicolumn{1}{c}{$\bm{89.3}\pm\bm{0.02}$} & \multicolumn{1}{c}{$\bm{87.6}\pm\bm{0.02}$} & \multicolumn{1}{c}{$\bm{86.6}\pm\bm{0.02}$} & \multicolumn{1}{c}{$\bm{85.8}\pm\bm{0.03}$} & \multicolumn{1}{c}{$\bm{84.5}\pm\bm{0.01}$} \\
		\midrule[0.6pt]

		\multirow{10}{*}{CIFAR100} & \multirow{5}{*}{Gaussian} & ERM & $68.2\pm0.01$ & $64.8\pm0.01$ & $60.6\pm0.01$ & $56.9\pm0.01$ & $53.9\pm0.01$ & $50.2\pm0.03$ \\
		&  & IRM & $64.7\pm0.02$ & $64.7\pm0.01$ & $62.2\pm0.01$ & $54.5\pm0.02$ & $53.4\pm0.03$ & $50.4\pm0.01$ \\
		&  & REx & $68.0\pm0.03$ & $65.1\pm0.03$ & $61.8\pm0.01$ & $56.8\pm0.01$ & $53.2\pm0.01$ & $51.5\pm0.01$ \\
		&  & GroupDRO & $66.1\pm0.01$ & $61.7\pm0.02$ & $59.3\pm0.03$ & $53.6\pm0.01$ & $54.0\pm0.02$ & $50.6\pm0.02$ \\
		&  & SharpDRO & $\bm{71.2}\pm\bm{0.02}$ & $\bm{70.1}\pm\bm{0.01}$ & $\bm{68.6}\pm\bm{0.01}$ & $\bm{58.8}\pm\bm{0.01}$ & $\bm{57.5}\pm\bm{0.02}$ & $\bm{53.8}\pm\bm{0.03}$
		\\ \cline{2-9} 
		
		& \multirow{5}{*}{JPEG} & ERM & $64.7\pm0.01$ & $62.8\pm0.02$ & $57.2\pm0.02$ & $54.7\pm0.03$ & $54.0\pm0.02$ & $50.6\pm0.03$ \\
		&  & IRM & $64.2\pm0.02$ & $62.8\pm0.04$ & $58.0\pm0.01$ & $56.3\pm0.02$ & $55.0\pm0.02$ & $51.8\pm0.03$ \\
		&  & REx & $63.5\pm0.03$ & $62.2\pm0.02$ & $58.4\pm0.03$ & $56.5\pm0.03$ & $55.1\pm0.02$ & $52.2\pm0.01$ \\
		&  & GroupDRO & $63.4\pm0.02$ & $61.6\pm0.04$ & $58.6\pm0.02$ & $57.2\pm0.02$ & $55.8\pm0.03$ & $53.1\pm0.02$ \\
		&  & SharpDRO & $\bm{65.3}\pm\bm{0.02}$ & $\bm{63.0}\pm\bm{0.02}$ & $\bm{59.8}\pm\bm{0.02}$ & $\bm{58.8}\pm\bm{0.02}$ & $\bm{57.8}\pm\bm{0.03}$ & $\bm{55.3}\pm\bm{0.03}$ \\ \midrule[0.6pt]

		\multirow{10}{*}{ImageNet30} & \multirow{5}{*}{Gaussian} & ERM & $87.5\pm0.01$ & $84.6\pm0.01$ & $81.9\pm0.01$ & $76.5\pm0.01$ & $71.2\pm0.01$ & $65.3\pm0.01$ \\ 
		&  & IRM & $86.6\pm0.01$ & $84.4\pm0.03$ & $80.6\pm0.01$ & $75.2\pm0.01$ & $70.7\pm0.03$ & $64.8\pm0.01$ \\ 
		&  & REx& $86.3\pm0.01$ & $83.8\pm0.03$ & $81.1\pm0.02$ & $75.6\pm0.02$ & $71.5\pm0.01$ & $66.1\pm0.03$ \\ 
		&  & GroupDRO & $85.1\pm0.02$ & $84.2\pm0.01$ & $81.2\pm0.03$ & $76.3\pm0.03$ & $72.0\pm0.02$ & $66.3\pm0.01$ \\ 
		&  & SharpDRO& $\bm{88.4}\pm\bm{0.02}$ & $\bm{87.6}\pm\bm{0.01}$ & $\bm{83.3}\pm\bm{0.01}$ & $\bm{79.1}\pm\bm{0.02}$ & $\bm{73.5}\pm\bm{0.03}$ & $\bm{68.7}\pm\bm{0.01}$ \\  \cline{2-9}

		& \multirow{5}{*}{JPEG} & ERM & $86.8\pm0.03$ & $85.3\pm0.03$ & $83.2\pm0.02$ & $82.6\pm0.01$ & $80.4\pm0.04$ & $78.2\pm0.02$ \\
		&  & IRM & $86.2\pm0.03$ & $85.1\pm0.02$ & $83.8\pm0.03$ & $83.2\pm0.03$ & $81.6\pm0.02$ & $79.1\pm0.01$ \\
		&  & REx & $85.8\pm0.02$ & $85.3\pm0.02$ & $83.5\pm0.02$ & $82.9\pm0.02$ & $81.4\pm0.02$ & $78.2\pm0.02$ \\
		&  & GroupDRO & $86.7\pm0.02$ & $84.9\pm0.02$ & $84.1\pm0.02$ & $84.5\pm0.02$ & $82.3\pm0.02$ & $79.0\pm0.02$ \\
		&  & SharpDRO & $\bm{87.4}\pm\bm{0.02}$ & $\bm{86.4}\pm\bm{0.03}$ & $\bm{86.2}\pm\bm{0.03}$ & $\bm{85.6}\pm\bm{0.02}$ & $\bm{83.9}\pm\bm{0.02}$ & $\bm{82.1}\pm\bm{0.03}$ \\  \bottomrule[1pt]
		
	\end{tabular}
\vspace{-2mm}
\end{table*}

\vspace{-2mm}
\begin{theorem}[\textbf{Informal}] Assuming the loss function $\mathbb{L}$ is $l$-Lipschitz smooth, satisfies $\mu$-Polyak-Łojasiewicz (PL) condition on the second variable $\omega$, and has unbiased estimation about the gradient as well as $\sigma^2$ bounded variance, we can get the convergence rate during $T$ iterations:
\begin{equation*}
\small
\begin{aligned}
\!\frac{1}{T}\!\sum_{t=0}^{T-1}\mathbb{E}\|\nabla\mathbb{L}^*(\theta_t)\|^2
&    \!\leq\! 320\sqrt{\frac{3\kappa^4l(\mathbb{E}[\mathbb{L}^*(\theta_0)]\!-\!\min_{\theta}\mathbb{E}[\mathbb{L}^*(\theta)])\sigma^2}{11MT}}\\
&=\mathcal{O}\left(\frac{\kappa^2}{\sqrt{MT}}\right),
\end{aligned}
\end{equation*}
where the conditional number $\kappa=l/\mu$ and $M$ means the sample batch(here we can choose $M=1$)\footnote{The resulting bound here means our SharpDRO can converge to the $\epsilon$-stationary point in $\frac{1}{\epsilon^2}$ iterations. Moreover, we will present our proof details in the \textbf{Appendix}, including the definitions, assumptions, and lemmas.}.
\end{theorem}

\section{Experiments}
\label{sec:experiments}
In our experiments, we first give details about our experimental setup. Then, we conduct quantitative experiments to compare to proposed SharpDRO with the most popular baseline methods by considering both distribution-aware and distribution agnostic settings, which shows the capability of SharpDRO to tackle the most challenging distributions. Finally, we conduct qualitative analyses to investigate the effectiveness of SharpDRO in achieving robust generalization.

\begin{table*}[t]
\vspace{-8mm}
	\scriptsize
	\centering
	\caption{Quantitative comparisons on distribution-agnostic robust generalization setting. Averaged accuracy ($\%$) with standard deviations are computed over three independent trails.}
		\vspace{-0.2cm}
	\setlength{\tabcolsep}{2.8mm}
	\label{tab:distribution_agnostic}
	\begin{tabular}{lllcccccc}
		\toprule[1pt]
		\multirow{2}{*}{Dataset} & \multirow{2}{*}{Type} & \multirow{2}{*}{Method} & \multicolumn{6}{c}{Corruption Severity} \\
		&  &  & \multicolumn{1}{c}{0} & \multicolumn{1}{c}{1} & \multicolumn{1}{c}{2} & \multicolumn{1}{c}{3} & \multicolumn{1}{c}{4} & \multicolumn{1}{c}{5} \\ \midrule[0.6pt]
		\multirow{6}{*}{CIFAR10} & \multirow{3}{*}{Gaussian} & JTT & \multicolumn{1}{c}{$89.9\pm0.02$} & \multicolumn{1}{c}{$88.8\pm0.02$} & \multicolumn{1}{c}{$86.5\pm0..0$2} & \multicolumn{1}{c}{$86.1\pm0.02$} & \multicolumn{1}{c}{$83.4\pm0.03$} & \multicolumn{1}{c}{$79.8\pm0.02$} \\
		&  & EIIL & \multicolumn{1}{c}{$88.6\pm0.02$} & \multicolumn{1}{c}{$87.5\pm0.03$} & \multicolumn{1}{c}{$86.3\pm0.03$} & \multicolumn{1}{c}{$85.4\pm0.02$} & \multicolumn{1}{c}{$83.2\pm0.03$} & \multicolumn{1}{c}{$78.8\pm0.01$} \\
		&  & SharpDRO & \multicolumn{1}{c}{$\bm{91.3}\pm\bm{0.01}$} & \multicolumn{1}{c}{$\bm{90.2}\pm\bm{0.02}$} & \multicolumn{1}{c}{$\bm{88.7}\pm\bm{0.01}$} & \multicolumn{1}{c}{$\bm{87.3}\pm\bm{0.02}$} & \multicolumn{1}{c}{$\bm{84.2}\pm\bm{0.02}$} & \multicolumn{1}{c}{$\bm{84.3}\pm\bm{0.01}$} \\ \cline{2-9}

		& \multirow{3}{*}{JPEG} & JTT & \multicolumn{1}{c}{$88.9\pm0.03$} & \multicolumn{1}{c}{$87.2\pm0.04$} & \multicolumn{1}{c}{$85.3\pm0.03$} & \multicolumn{1}{c}{$83.5\pm0.03$} & \multicolumn{1}{c}{$81.0\pm0.04$} & \multicolumn{1}{c}{$78.8\pm0.03$} \\
		&  & EIIL & \multicolumn{1}{c}{$89.4\pm0.03$} & \multicolumn{1}{c}{$87.6\pm0.03$} & \multicolumn{1}{c}{$85.3\pm0.04$} & \multicolumn{1}{c}{$83.2\pm0.03$} & \multicolumn{1}{c}{$81.6\pm0.02$} & \multicolumn{1}{c}{$79.1\pm0.01$} \\
		&  & SharpDRO & \multicolumn{1}{c}{$\bm{90.3}\pm\bm{0.02}$} & \multicolumn{1}{c}{$\bm{88.2}\pm\bm{0.02}$} & \multicolumn{1}{c}{$\bm{87.2}\pm\bm{0.02}$} & \multicolumn{1}{c}{$\bm{85.3}\pm\bm{0.02}$} & \multicolumn{1}{c}{$\bm{84.2}\pm\bm{0.02}$} & \multicolumn{1}{c}{$\bm{82.3}\pm\bm{0.03}$} \\  \midrule[0.6pt]

		\multirow{6}{*}{CIFAR100} & \multirow{3}{*}{Gaussian} & JTT & $68.0\pm0.02$ & $65.3\pm0.02$ & $61.3\pm0.01$ & $56.3\pm0.01$ & $54.2\pm0.03$ & $51.2\pm0.02$ \\
		&  & EIIL & $67.2\pm0.01$ & $66.2\pm0.02$ & $61.0\pm0.02$ & $55.8\pm0.02$ & $54.6\pm0.03$ & $52.1\pm0.02$ \\
		&  & SharpDRO & $\bm{69.6}\pm\bm{0.03}$ & $\bm{68.0}\pm\bm{0.02}$ & $\bm{63.6}\pm\bm{0.03}$ & $\bm{58.2}\pm\bm{0.02}$ & $\bm{56.5}\pm\bm{0.03}$ & $\bm{54.1}\pm\bm{0.03}$ \\ \cline{2-9}

		& \multirow{3}{*}{JPEG} & JTT & $63.1\pm0.03$ & $61.0\pm0.04$ & $58.4\pm0.03$ & $56.2\pm0.02$ & $54.5\pm0.04$ & $51.9\pm0.03$ \\
		&  & EIIL & $63.6\pm0.02$ & $61.5\pm0.03$ & $58.8\pm0.03$ & $\bm{57.7}\pm\bm{0.02}$ & $54.2\pm0.03$ & $52.2\pm0.04$ \\
		&  & SharpDRO & $\bm{64.2}\pm\bm{0.02}$ & $\bm{62.5}\pm\bm{0.03}$ & $\bm{60.1}\pm\bm{0.03}$ & $57.6\pm0.03$ & $\bm{57.2}\pm\bm{0.02}$ & $\bm{54.8}\pm\bm{0.02}$ \\  \midrule[0.6pt]

		\multirow{6}{*}{ImageNet30} & \multirow{3}{*}{Gaussian} & JTT & $87.3\pm0.02$ & $84.5\pm0.02$ & $82.3\pm0.04$ & $75.6\pm0.01$ & $72.1\pm0.04$ & $66.5\pm0.02$ \\
		&  & EIIL & $\bm{88.2}\pm\bm{0.02}$ & $85.2\pm0.03$ & $81.3\pm0.02$ & $74.5\pm0.02$ & $71.5\pm0.02$ & $65.0\pm0.04$ \\
		&  & SharpDRO & $87.1\pm0.02$ & $\bm{86.9}\pm\bm{0.02}$ & $\bm{83.5}\pm\bm{0.03}$ & $\bm{78.0}\pm\bm{0.02}$ & $\bm{74.9}\pm\bm{0.02}$ & $\bm{68.4}\pm\bm{0.03}$ \\ \cline{2-9}

		& \multirow{3}{*}{JPEG} & JTT & $85.4\pm0.02$ & $84.3\pm0.02$ & $82.5\pm0.02$ & $80.0\pm0.02$ & $78.8\pm0.03$ & $77.4\pm0.02$ \\
		&  & EIIL & $85.3\pm0.02$ & $84.2\pm0.03$ & $83.2\pm0.03$ & $80.6\pm0.02$ & $78.4\pm0.02$ & $77.2\pm0.02$ \\
		&  & SharpDRO & $\bm{86.6}\pm\bm{0.03}$ & $\bm{85.7}\pm\bm{0.02}$ & $\bm{85.0}\pm\bm{0.02}$ & $\bm{84.3}\pm\bm{0.02}$ & $\bm{83.2}\pm\bm{0.03}$ & $\bm{82.5}\pm\bm{0.02}$ \\  \bottomrule[1pt]
		
	\end{tabular}
\vspace{-2mm}
\end{table*}

\subsection{Experimental Setup}
For distribution-aware situation, we choose GroupDRO~\cite{sagawa2019distributionally}, IRM~\cite{arjovsky2019invariant}, REx~\cite{krueger2021out}, and ERM for comparisons. As for distribution-agnostic situation, we pick JTT~\cite{liu2021just} and Environment Inference for Invariant Learning (EIIL)~\cite{creager2021environment} for baseline methods\footnote{Note that we do not include sharpness minimization method SAM~\cite{foret2020sharpness} in this problem setting because its OOD generalization performance is worse than ERM. However, we conduct detailed analysis between SharpDRO and SAM in Section~\ref{sec:qualitative_analysis}}. For each problem setting, we construct corrupted dataset using CIFAR10/100~\cite{krizhevsky2009learning} and ImageNet30~\cite{russakovsky2015imagenet} datasets. Specifically, we following~\cite{hendrycks2019benchmarking} to perturb the image data with severity level varies from $1$ to $5$ by using two types of corruption: ``Gaussian Noise'' and ``JPEG Compression''. Moreover, the clean data are considered as having a corruption severity of $0$. For each corrupted distribution, we sample them with different probabilities by following Poisson distribution $P(s; \lambda=1)$, \textit{i.e.}, for $s$ varies from $0$ to $5$, the sample probabilities are $\left\{0.367, 0.367, 0.184, 0.061, 0.015, 0.003\right\}$, respectively. Then, we test the robust performance on each data distribution. For hyper-parameter $\rho$, we follow~\cite{foret2020sharpness} by setting it to $0.05$ to control the magnitude of $\epsilon^*$. For each experiment, we conduct three independent trials and report the average test accuracy with standard deviations. Please see more details on experimental setting, results on other corruptions, and practical implementation in the \textbf{Appendix}.

\subsection{Quantitative Comparisons}
In this part, we focus on three questions: 1) Can SharpDRO perform well on two situations of robust generalization? 2) Does SharpDRO generalize well on the most severely corrupted distributions? and 3) Is SharpDRO able to tackle different types of corruption? To answer these questions, we conduct experiments on both two settings by testing on different corruption types and severity levels.

\paragraph{Distribution-Aware Robust Generalization}
\label{sec:distribution_aware}
As shown in Table~\ref{tab:distribution_aware}, we can see that SharpDRO surpasses other methods with larger performance gains as the corruption severity goes stronger. Especially in CIFAR10 dataset on ``Gaussian Noise'' corruption, improvement margin between SharpDRO and second-best method is $1.1\%$ with severity of $0$, which is further increased to about $\bm{4.2\%}$ with severity of $5$, which indicates the capability of SharpDRO on generalization against severe corruptions. Moreover, SharpDRO frequently outperforms other methods on all scenarios, which manifests the robustness of SharpDRO against various corruption types.

\paragraph{Distribution-Agnostic Robust Generalization}
\label{sec:distribution_agnostic}
In Table~\ref{tab:distribution_agnostic}, we can see a similar phenomenon as in Table~\ref{tab:distribution_aware} that the more severe corruptions are applied, the larger performance gains SharpDRO achieves. Especially, in the ImageNet30 dataset corrupted by ``JPEG Compression'', SharpDRO shows about $1.2\%$ performance gains upon the second-best method with severity $0$, which is further increased to almost $\bm{5.1\%}$ with severity $5$. Moreover, SharpDRO is general to all three corruption types, as it surpasses other methods in most cases. Therefore, the proposed method can perfectly generalize to worst-case data even without the distribution annotations.

\begin{figure*}[t]
	\vspace{-6mm}
	\centering
	\begin{minipage}[t]{0.195\textwidth}
		\centering
		\includegraphics[width=\linewidth]{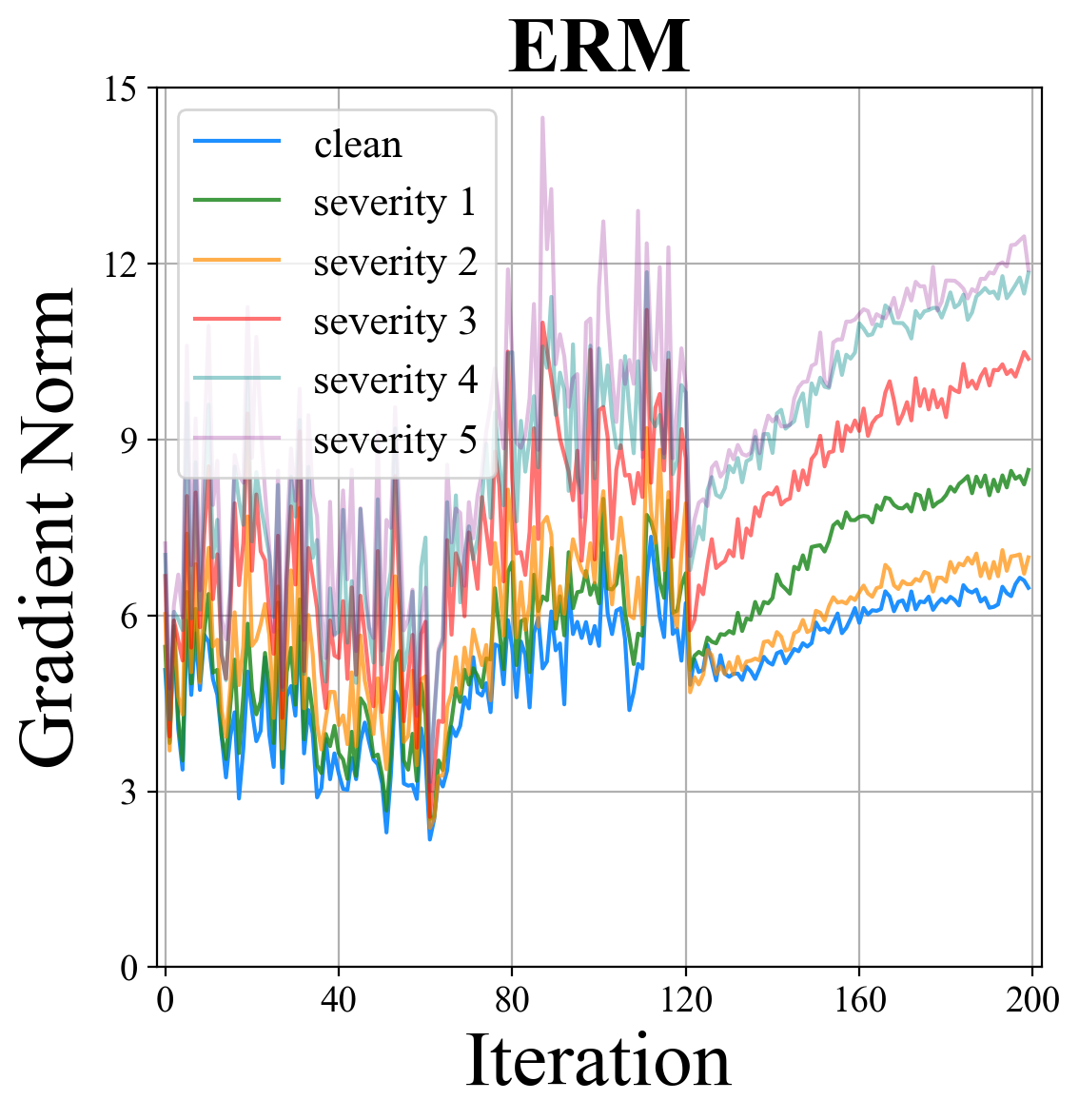}
	\end{minipage}
	\begin{minipage}[t]{0.195\textwidth}
		\centering
		\includegraphics[width=\linewidth]{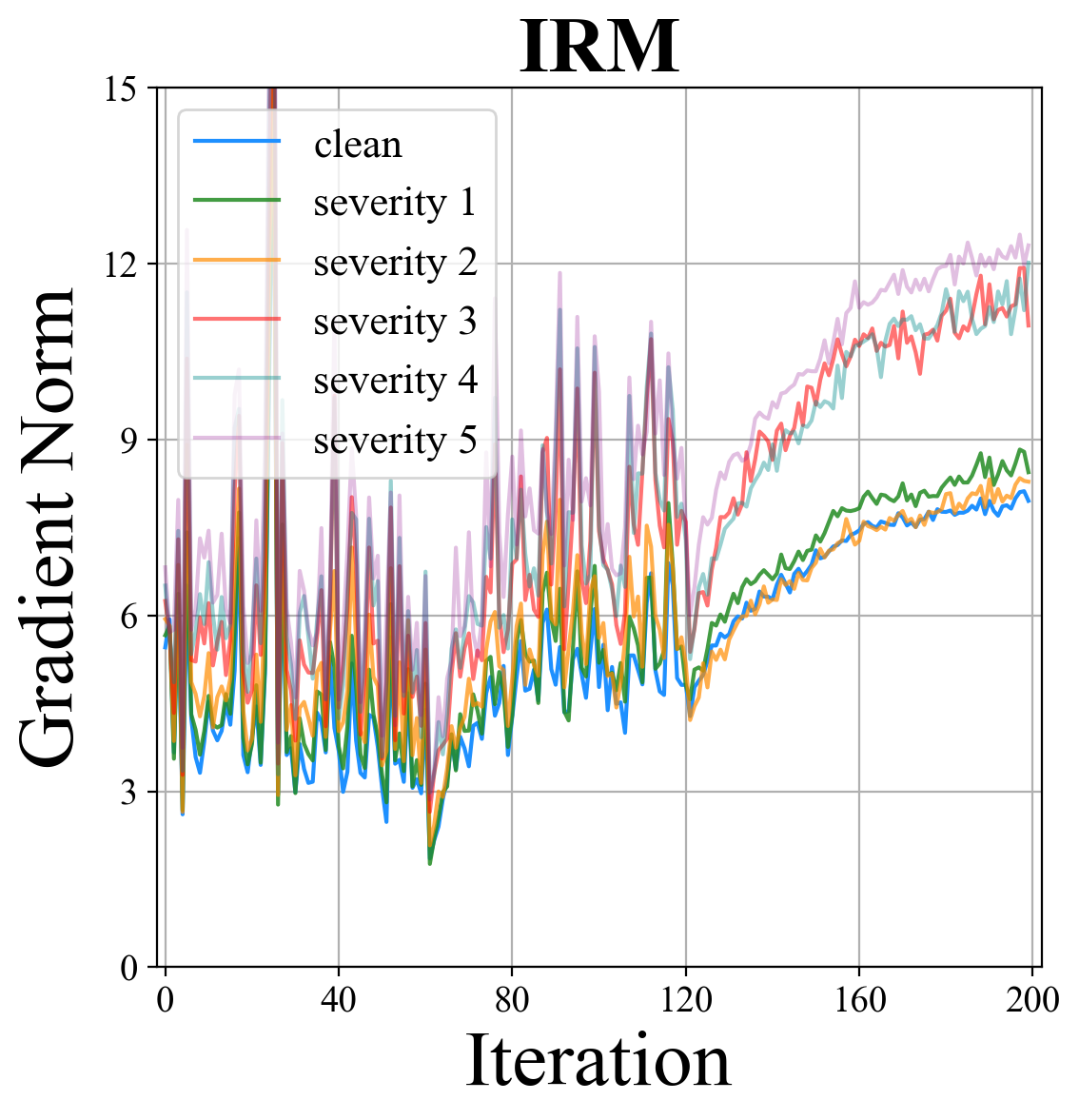}
	\end{minipage}
	\begin{minipage}[t]{0.195\textwidth}
		\centering
		\includegraphics[width=\linewidth]{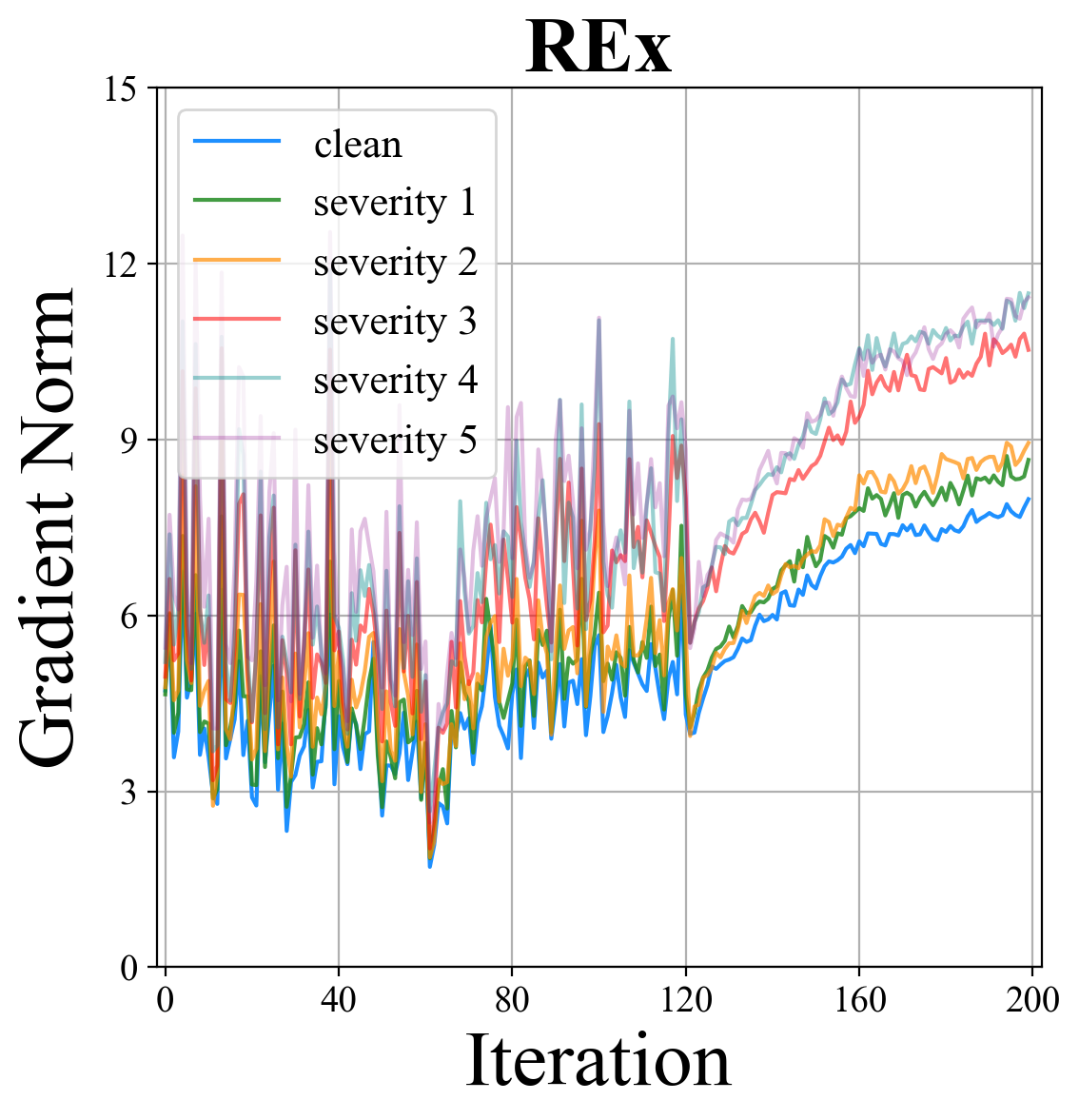}
	\end{minipage}
	\begin{minipage}[t]{0.195\textwidth}
		\centering
		\includegraphics[width=\linewidth]{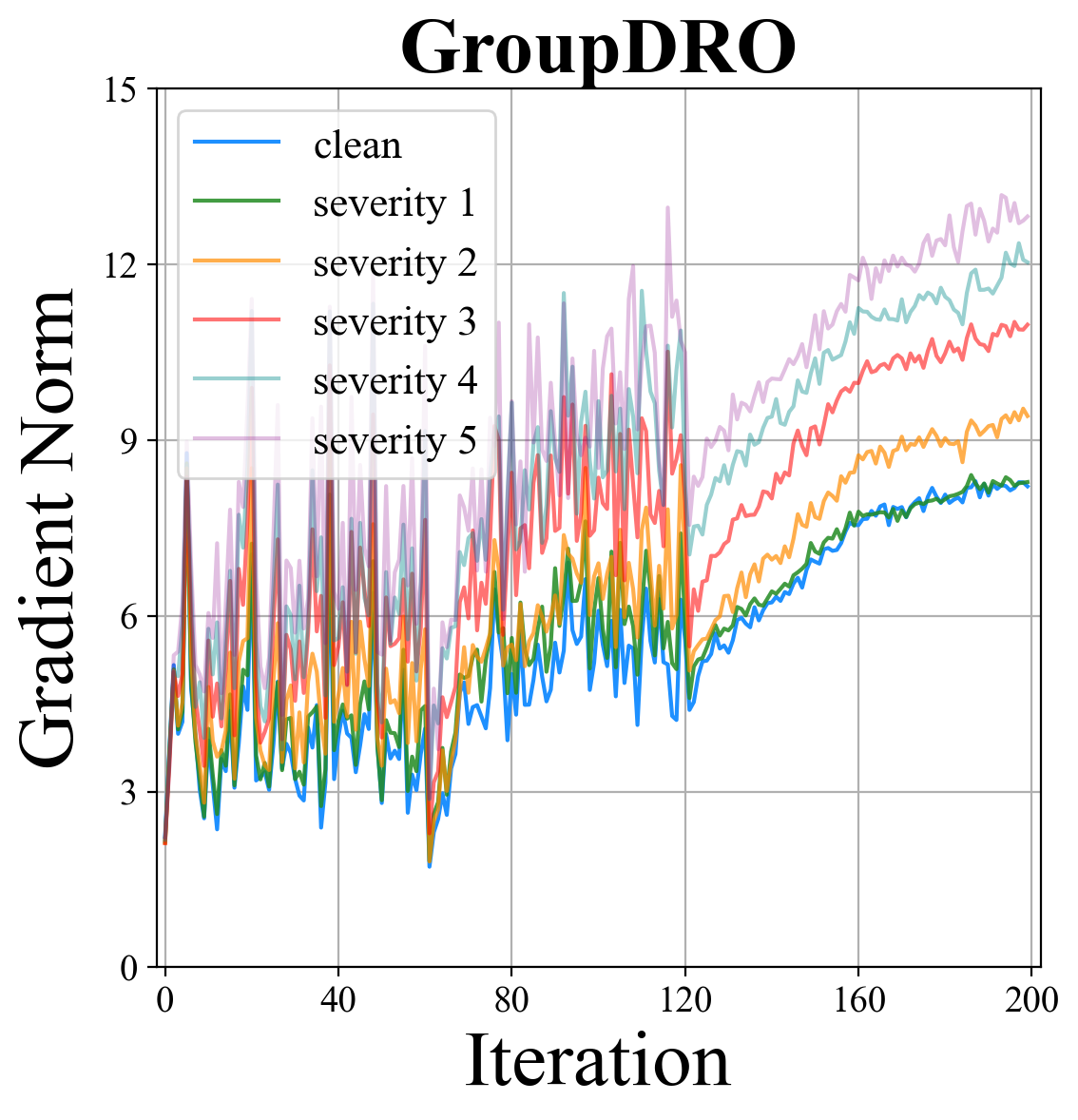}
	\end{minipage}
	\begin{minipage}[t]{0.195\textwidth}
		\centering
		\includegraphics[width=\linewidth]{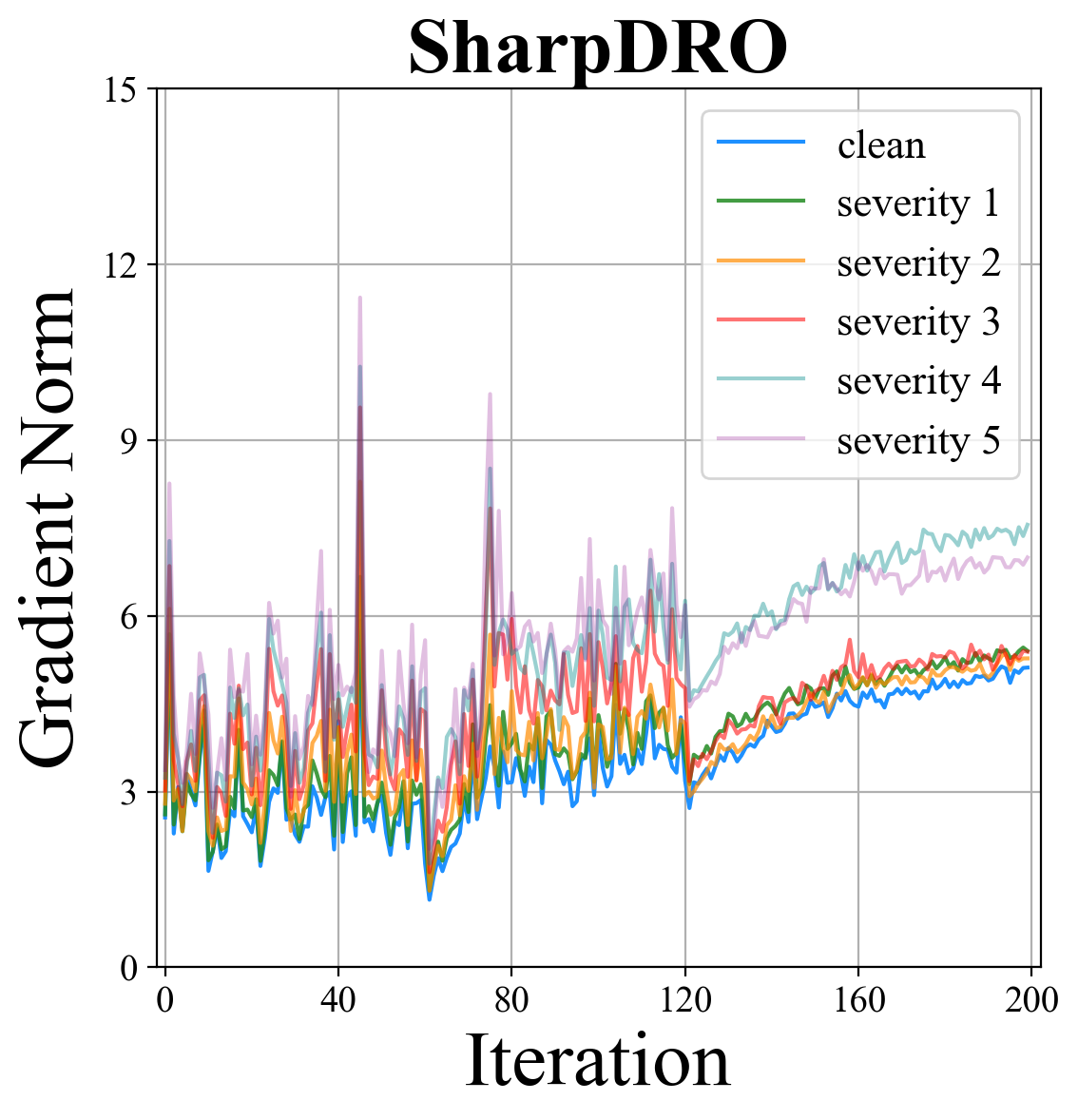}
	\end{minipage}
	\vspace{-3mm}
	\caption{\small Gradient norm comparisons between different methods over all corrupted distributions.}
	\label{fig:gradnorm}
	\vspace{-0.4cm}
\end{figure*}

\subsection{Qualitative Analysis}
\label{sec:qualitative_analysis}
To investigate the effectiveness of SharpDRO, we first conduct an ablation study to show that the worst-case sharpness minimization is essential for achieving generalization with robustness. Then, we utilize gradient norm, an important criterion to present training stability, to validate that our method is stable for severely corrupted distributions. Then, we analyze the hyper-parameter $\rho$ and OOD score $\bar{w}$ to disclose the effectiveness of sharpness minimization and worst-case data selection. Finally, another second-order methdo SAM~\cite{foret2020sharpness, zhong2022improving, mi2022make, sun2023adasam, sun2023fedspeed} is investigated to discover the efficiency property of SharpDRO. All analyses are conducted using CIFAR10 with ``Gaussian Noise'' corruption.

\vspace{-3mm}
\begin{table}[H]
	\small
	\caption{Ablation study. "w/o data selection" denotes training without worst-case data selection, which recovers SAM~\cite{foret2020sharpness}, and "w/o sharp min" indicates training without sharpness minimization, which is the same as GroupDRO~\cite{sagawa2019distributionally}.}
		\vspace{-0.2cm}
	\setlength{\tabcolsep}{1.05mm}
	\label{tab:ablation_study}
	\begin{tabular}{lclllll}
		\toprule[1pt]
		\multirow{2}{*}{Method} & \multicolumn{6}{c}{Corruption Severity} \\
		& 0 & \multicolumn{1}{c}{1} & \multicolumn{1}{c}{2} & \multicolumn{1}{c}{3} & \multicolumn{1}{c}{4} & \multicolumn{1}{c}{5} \\ \midrule[0.6pt]
		w/o data selection (SAM) & 93.2 & \multicolumn{1}{c}{90.5} & \multicolumn{1}{c}{87.6} & \multicolumn{1}{c}{82.1} & \multicolumn{1}{c}{80.5} & \multicolumn{1}{c}{75.4} \\
		w/o sharp min (GroupDRO) & \multicolumn{1}{l}{90.2} & 89.1 & 88.4 & 84.3 & 83.0 & 78.2 \\
		\rowcolor{gray!25} SharpDRO & \multicolumn{1}{l}{92.9} &  91.3 & 90.5 & 88.4 & 86.9 & 84.7 \\ \bottomrule[1pt]
	\end{tabular}
\vspace{-6mm}
\end{table}

\vspace{-2mm}
\paragraph{Ablation Study}
\label{sec:ablation_study}
By eliminating the worst-case data selection, we recover the original sharpness minimization method SAM~\cite{foret2020sharpness}. Then, we remove the sharpness minimization module, which is basically training via GroupDRO. The ablation results are shown in Table~\ref{tab:ablation_study}. We can see that deploying SAM on the whole training dataset can achieve improved results on the clean dataset. However, the robust performance on corrupted distributions are even worse than GroupDRO. This could be because that sharpness is easy to be dominated by principle distributions, which is misleading for generalization to small distributions. Thus, the sharpness of corrupted data would be sub-optimal. As for GroupDRO, it fails to produce a flat loss surface for worst-case data, hence cannot generalize as well as the proposed SharpDRO.

\vspace{-1mm}
\paragraph{Distributional Stability}
To show our method can be stable even in the most challenging distributions, we show the gradient norm on a validation set including corruption severity from $0$ to $5$. As shown in Figure~\ref{fig:gradnorm}, SharpDRO not only produces the smallest norm value but also can ensure almost equal gradient norm across all corrupted distributions, which indicates that SharpDRO is the most distributionally stable method among all compared methods.

\paragraph{Parameter Analysis}

\begin{figure}[t]
	\centering
    \includegraphics[width=\linewidth]{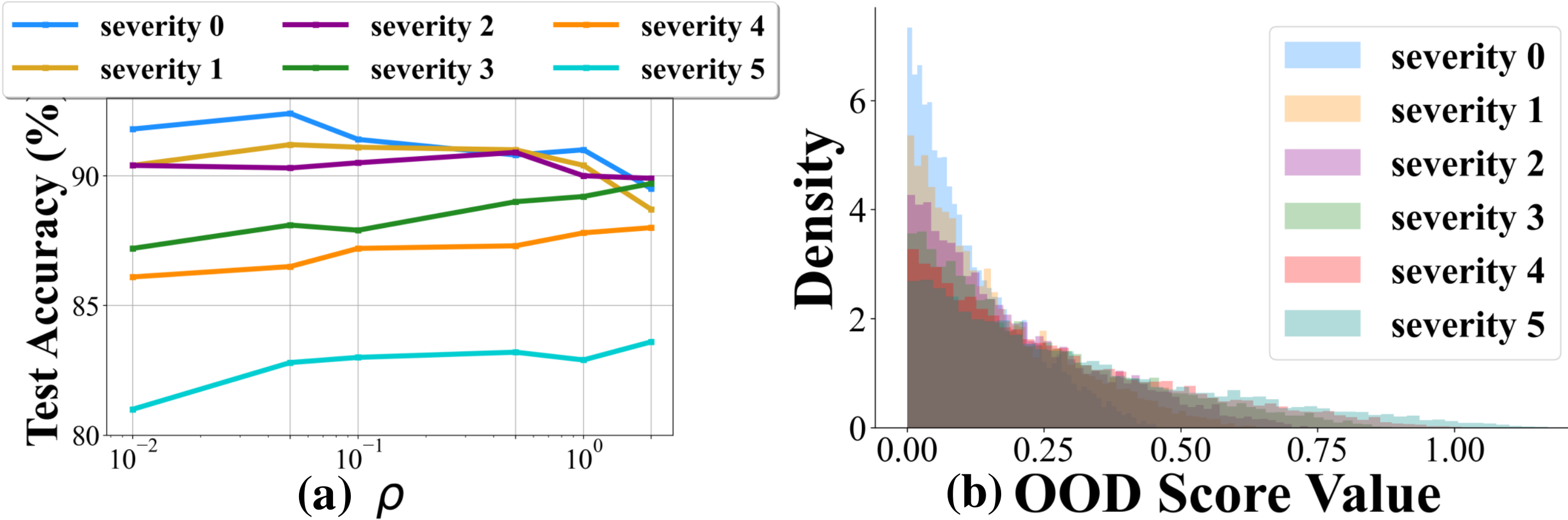}
\caption{\small (a) Sensitivity of $\rho$ whose value is set to $\left\{0.01, 0.05, 0.1, 0.5, 1, 2\right\}$. (b) Distribution of the normalized OOD score $\bar{\omega}$ on distribution $s=0$ to $5$.}
\label{fig:sens_ood}
\vspace{-0.4cm}
\end{figure}

To understand how the scale parameter $\rho$ affects our generalization performance, we conduct sensitivity analysis by changing this value and show the test results of different distributions. In figure~\ref{fig:sens_ood} (a), we find an interesting discovery that as $\rho$ increases, which indicates the perturbation magnitude $\epsilon^*$ enlarges, would enhance the generalization of severely corrupted data but degrades the performance of slightly corrupted data. This might be because the exploration of hard distributions needs to cover wide range of neighborhood to ensure generalization. On the contrary, exploration too far on easy distributions can reach out-of-distribution, thus causing performance degradation. Therefore, for practitioners who aim to generalize on small and difficult datasets, we might be able to enhance performance by aggressively setting a large perturbation scale.

\paragraph{OOD Score Analysis}
\label{sec:ood_score}
The OOD score is leveraged to select worst-case data for the distribution-agnostic setting. To show its effectiveness in selecting the noisy data, we plot the value distribution of OOD scores from all corrupted distributions in epoch $30$ in Figure~\ref{fig:sens_ood} (b). We can see the tendency that a severer corruption has larger OOD scores. Therefore, our OOD score is a valid criterion to select worst-case data. Note that during the training process, the worst-case data would be \textbf{gradually learned}, thus the OOD score can become \textbf{smaller}, which explains why the value distribution of our score is not as separable as OOD detection does.

\paragraph{Training Efficiency Analysis}
It is clear that the proposed SharpDRO method is a second-order optimization method. Hence, when compared to first-order methods such as GroupDRO and REx, computational cost is the price to pay for achieving improved generalization performance\footnote{Note that our method can be deployed with existing efficient sharpness-based methods~\cite{zhang2022ga, du2022efficient, du2022sharpness, zhao2022ss}.}. However, to further explore the advantage of SharpDRO compared to other second-order method, here we use SAM~\cite{foret2020sharpness} as a competitor, and show their computational time as well as worst-case accuracy ($s=5$) in Figure.~\ref{fig:time}. We can see that on all three datasets, our SharpDRO requires nearly the same time to train, and significantly outperforms the worst-case performance of SAM, owing to our efficient worst-case data selection which is vital for robust generalization against severe corruptions.

\begin{figure}[t]
	\centering
	\includegraphics[width=0.85\linewidth]{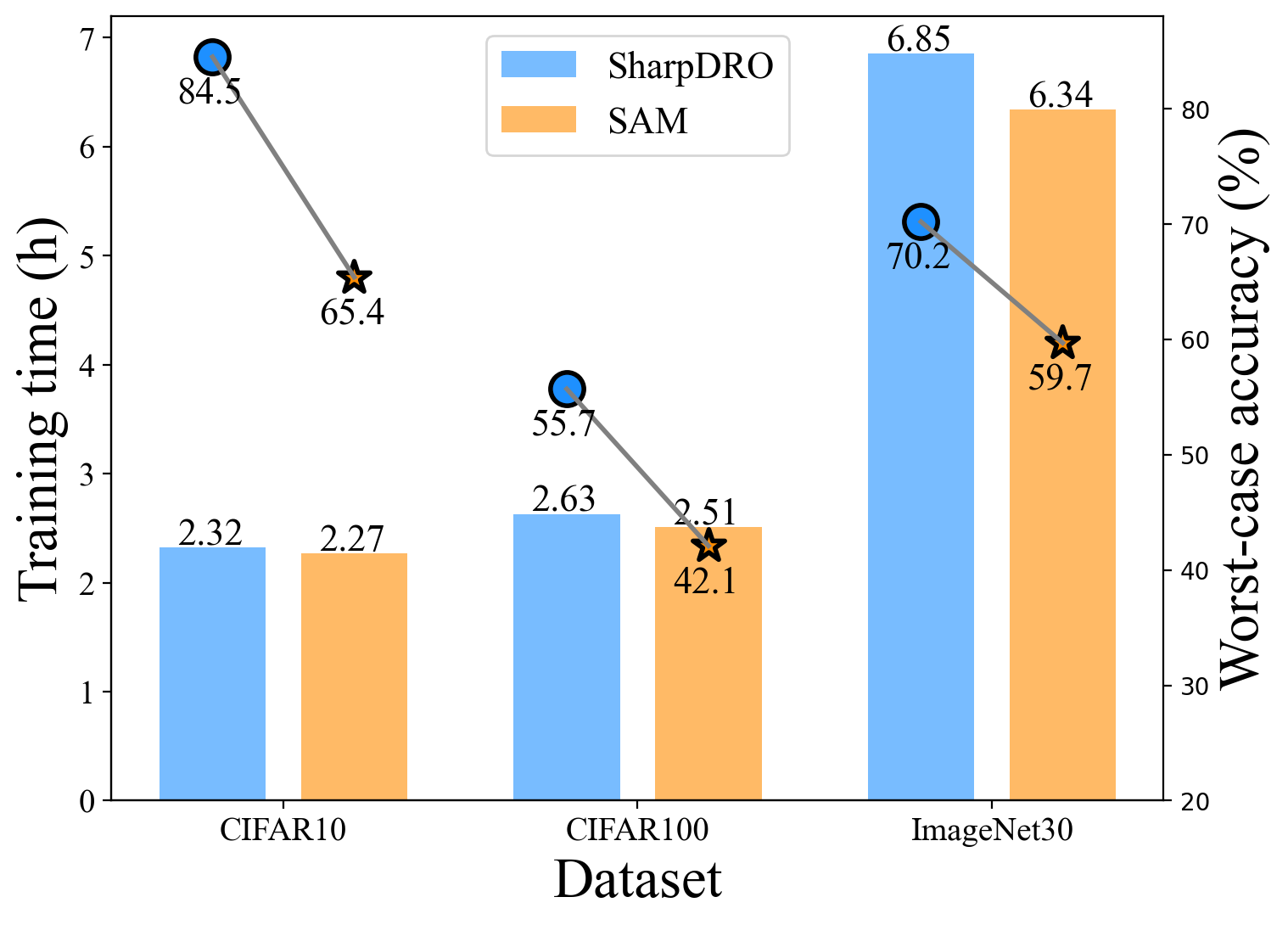}
	\vspace{-4mm}
	\caption{\small Efficiency comparison between SharpDRO and SAM.}
	\label{fig:time}
	\vspace{-4mm}
\end{figure}

\vspace{-1mm}
\section{Conclusion}
\vspace{-1mm}
\label{sec:conclusion}
In this paper, we proposed a SharpDRO approach to enhance the generalization performance of DRO methods. Specifically, we focus on minimizing the sharpness of worst-case data to learn flat loss surfaces. As a result, SharpDRO is more robust to severe corruptions compared to other methods. Moreover, we apply SharpDRO to distribution-aware and distribution-agnostic settings and proposed an OOD detection process to select the worst-case data when the distribution index is not known. Extensive quantitative and qualitative experiments have been conducted to show that SharpDRO can deal with the most challenging corrupted distributions and achieve improved generalization results compared to well-known baseline methods.

%

{\small
	\bibliographystyle{ieee_fullname}
	\bibliography{egbib}
}

\appendix
\onecolumn
\vspace{0.5in}
\begin{center}
	\rule{6.875in}{0.7pt}\\ 
	{\Large\bf Supplementary Material for ``Robust Generalization against Photon-Limited \\ Corruptions via Worst-Case Sharpness Minimization''}
	\rule{6.875in}{0.7pt}
\end{center}

In the supplementary material, we first provide the details in the proof for the theoretical result in the main paper in Section.~\ref{appendix_proof}. Then, we give details about our implementation details in Section~\ref{appendix_details}. Finally, we show more experimental results using different types of corruptions in Section~\ref{appendix_exp_results}.
 
\section{Convergence Analyses} \label{appendix_proof}
\newtheorem{assumption}{Assumption}[section]
\newtheorem{definition}{Definition}[section]
\newtheorem{lemma}{Lemma}[section]

\subsection{Preliminaries}
We first give some notations before we start our proof for the convergence.
\begin{enumerate}
    \item We denote the expectation value for the loss function as $\mathbb{L}(\theta,\omega):=\mathbb{E}_{(x,y)\sim Q}\mathcal{L}(\theta,\omega;(x,y))$, and so as the SAM function that $\mathbb{R}(\theta,\omega)=\mathbb{E}_{(x,y)\sim Q}R(\theta,\omega;(x,y))$. So our objective can be turned into: $\min_{\theta}\{\max_{\omega}\mathbb{L}(\theta,\omega)\}+\mathbb{R}(\theta,\omega)$. And recalling our SharpDRO algorithm, we restate the meaning of the parameters: the model is parameterized by $\theta$ and $\omega$ means the weighted sampling.
    \item $\kappa$ is the condition number that $\kappa=\frac{l}{\mu}$, where $l$ is the Lipschitz-smoothness in Assumption \ref{as:smooth} and $\mu$ means the PL condition in Assumption \ref{as:pl}.
    \item We define $\mathbb{L}^*(\theta)=\max_{\omega}\mathbb{L}(\theta,\omega)$ and $\omega^*(\theta)=arg\max_{\omega}\mathbb{L}(\theta,\omega)$.
\end{enumerate}

\subsection{Update Rule}
Before our theoretical analyses, we need to make the update rule for each variable explicit. We have to pay attention to the fact that our algorithm is stochastic that we can not directly get the real value of the gradient $\nabla\mathbb{L}(\theta,\omega)$, rather we estimate it by batches of samples $g_{\theta}(\theta,\omega)=\frac{1}{M}\sum_{i=1}^M\frac{\partial \mathcal{L}}{\partial \theta}(\theta,\omega;(x_i,y_i))$ and $g_{\omega}(\theta,\omega)=\frac{1}{M}\sum_{i=1}^M\frac{\partial \mathcal{L}}{\partial \omega}(\theta,\omega;(x_i,y_i))$, who hold some properties we will introduce in Assumption \ref{as:bv}. So the optimization iteration is executed as follows in reality:
\begin{equation}
    \begin{aligned}
    &\theta_{t+1}=\theta_t-\eta_{\theta}g_{\theta}(\theta_t+\rho g_{\theta}(\theta_t,\omega_t),\omega_t);\\
    &\omega_{t+1}=\omega_t+\eta_{\omega}\nabla_{\omega}g_{\omega}(\theta_t,\omega_t).
    \end{aligned}
\end{equation}

We further give a notation for brief that $\theta_{t+1/2}\triangleq\theta_t+\rho g_{\theta}(\theta_t,\omega_t)$, so the update for $\theta$ can be simplified as: $\theta_{t+1}=\theta_t-\eta_{\theta}g_{\theta}(\theta_{t+1/2},\omega_t)$.

\subsection{Assumptions}
We also have to make some necessary assumptions on our problem setting for this convergence proof:
\begin{assumption}[Bounded variance]\label{as:bv}
The unbiased estimation about the gradient of the loss function also has bounded variance that:
\begin{equation*}
    \begin{aligned}
    \mathbb{E}_{(x,y)\sim Q}[\frac{\partial\mathcal{L}}{\partial\theta}(\theta,\omega;(x,y))]=\nabla_{\theta}\mathbb{L}(\theta,\omega),&\quad \mathbb{E}_{(x,y)\sim Q}\|\frac{\partial \mathcal{L}}{\partial \theta}(\theta,\omega;(x,y))-\nabla_{\theta}\mathbb{L}(\theta,\omega)\|^2\leq\sigma^2;\\
    \mathbb{E}_{(x,y)\sim Q}[\frac{\partial\mathcal{L}}{\partial\omega}(\theta,\omega;(x,y))]=\nabla_{\omega}\mathbb{L}(\theta,\omega),&\quad \mathbb{E}_{(x,y)\sim Q}\|\frac{\partial \mathcal{L}}{\partial \omega}(\theta,\omega;(x,y))-\nabla_{\omega}\mathbb{L}(\theta,\omega)\|^2\leq\sigma^2.
    \end{aligned}
\end{equation*}
\end{assumption}
\begin{remark}
Since $g_{\theta}$ and $g_{\omega}$ are the averaged samples that: $g_{\theta}=\frac{1}{M}\sum_{i=1}^M \frac{\partial \mathcal{L}}{\partial \theta }(\theta,\omega;(x_i,y_i))$ and $g_{\omega}=\frac{1}{M}\sum_{i=1}^M\frac{\partial\mathcal{L}}{\partial\omega}(\theta,\omega;(x_i,y_i))$ respectively, they also have the unbiased property and have bounded variance:
\begin{equation*}
    \begin{aligned}
    \mathbb{E}_{(x,y)\sim Q}[g_{\theta}(\theta,\omega;(x,y))]=\nabla_{\theta}\mathbb{L}(\theta,\omega),&\quad \mathbb{E}_{(x,y)\sim Q}\|g_{\theta}(\theta,\omega;(x,y))-\nabla_{\theta}\mathbb{L}(\theta,\omega)\|^2\leq\frac{\sigma^2}{M};\\
    \mathbb{E}_{(x,y)\sim Q}[g_{\omega}(\theta,\omega;(x,y))]=\nabla_{\omega}\mathbb{L}(\theta,\omega),&\quad \mathbb{E}_{(x,y)\sim Q}\|g_{\omega}(\theta,\omega;(x,y))-\nabla_{\omega}\mathbb{L}(\theta,\omega)\|^2\leq\frac{\sigma^2}{M}.
    \end{aligned}
\end{equation*}
\end{remark}

\begin{assumption}[Lipschitz smooth]\label{as:smooth}
$\mathcal{L}(\theta,\omega;(x,y))$ is differential and $l$-Lipschitz smooth for every given sample $(x,y)$:
\begin{equation*}
    \begin{aligned}
    \|\nabla_{\theta}\mathcal{L}(\theta_1,\omega;(x,y))-\nabla_{\theta}\mathcal{L}(\theta_2,\omega;(x,y))\|&\leq l\|\theta_1-\theta_2\|,\quad\forall \omega, (x,y);\\
    \|\nabla_{\omega}\mathcal{L}(\theta,\omega_1;(x,y))-\nabla_{\omega}\mathcal{L}(\theta,\omega_2;(x,y))\|&\leq l\|\omega_1-\omega_2\|,\quad\forall \theta, (x,y).
    \end{aligned}
\end{equation*}
\end{assumption}
\begin{remark}
So the expectation function $\mathbb{L}$ also have the Lipschitz smooth property that:
\begin{equation*}
    \begin{aligned}
    \|\nabla_{\theta}\mathbb{L}(\theta_1,\omega)-\nabla_{\theta}\mathbb{L}(\theta_2,\omega)\|&\leq\mathbb{E}\|\nabla_{\theta}\mathcal{L}(\theta_1,\omega;(x,y))-\nabla_{\theta}\mathcal{L}(\theta_2,\omega;(x,y))\|\leq l\|\theta_1-\theta_2\|,\quad\forall \omega;\\
    \|\nabla_{\omega}\mathbb{L}(\theta,\omega_1)-\nabla_{\omega}\mathbb{L}(\theta,\omega_2)\|&\leq \mathbb{E}\|\nabla_{\omega}\mathcal{L}(\theta,\omega_1;(x,y))-\nabla_{\omega}\mathcal{L}(\theta,\omega_2;(x,y))\|\leq l\|\omega_1-\omega_2\|,\quad\forall \theta.
    \end{aligned}
\end{equation*}
\end{remark}

\begin{assumption}[PL condition]\label{as:pl}
The loss function $\mathbb{L}(\theta,\cdot)$ satisfies PL condition on every given $\theta$, i.e., there exists $\mu>0$ such that $\|\nabla_{\omega}\mathbb{L}(\theta,\omega)\|^2\geq2\mu[\max_{\omega}\mathbb{L}(\theta,\omega)-\mathbb{L}(\theta,\omega)], \forall \theta,\omega$.
\end{assumption}

\subsection{Useful Lemmas}
In this part, we will prove some necessary lemmas for us to prove the convergence bound. And we will give the definition of the stationary point of our problem at the beginning.

\begin{definition}[Stationary measure]
$\theta$ is defined as the $\epsilon$-stationary point of our problem if $\mathbb{E}\|\nabla\mathbb{L}^*(\theta)\|\leq\epsilon$ for any $\epsilon\geq0$.
\end{definition}
\begin{remark}
For minmax problem, there are usually two ways to measure the stationary point. The other one is measured two-side that: when $\mathbb{E}\|\nabla_{\theta}\mathbb{L}(\theta,\omega)\|\leq\epsilon$ and $\mathbb{E}\|\nabla_{\omega}\mathbb{L}(\theta,\omega)\|\leq\epsilon$, we claim $(\theta,\omega)$ is the $(\epsilon,\epsilon)$-stationary point. It has been proved in \cite{yang2022faster} that these two measures can be translated into each other when $\mathbb{L}^*$ is smooth which will be shown in Lemma \ref{le:L*smooth}. But what we compute is the model parameter $\theta$ using the algorithm SharpDRO. So we choose the measure by $\mathbb{E}\|\mathbb{L}^*(\theta)\|$ here.
\end{remark}

\begin{lemma}~\cite{nouiehed2019solving}\label{le:L*smooth}
Under Assumption \ref{as:smooth} and \ref{as:pl}, $\mathbb{L}^*(\theta)$ is $(l+\frac{l^2}{2\mu})$-Lipschitz smooth with the gradient:
$$\nabla_{\theta}\mathbb{L}^*(\theta,\omega)=\nabla_{\theta}\mathbb{L}(\theta,\omega^*(\theta)).$$
\end{lemma}

\begin{lemma}~\cite{nouiehed2019solving}\label{le:w^*smooth}
Under Assumption \ref{as:smooth} and \ref{as:pl}, $\omega^*(\cdot)$ is smooth about its variable:
$$\|\omega^*(\theta_1)-\omega^*(\theta_2)\|\leq\frac{l}{2\mu}\|\theta_1-\theta_2\|, \quad\forall \theta_1,\theta_2.$$
\end{lemma}

\begin{lemma}\label{eq:gradientbound}
We give an estimation that $\mathbb{E}\|g_{\theta}(\theta_{t+1/2},\omega_t)\|^2\leq(4\rho^2l^2+2\rho l+2)\mathbb{E}\|\nabla_{\theta}\mathbb{L}(\theta_t,\omega_t)\|^2+(5\rho^2l^2+2)\frac{\sigma^2}{M}$.
\end{lemma}
\begin{proof}
\begin{equation}
    \begin{aligned}
    \mathbb{E}\|g_{\theta}(\theta_{t+1/2},\omega_t)\|^2=-\mathbb{E}\|\nabla_{\theta}\mathbb{L}(\theta_t,\omega_t)\|^2+\mathbb{E}\|g_{\theta}(\theta_{t+1/2},\omega_t)-\nabla_{\theta}\mathbb{L}(\theta_t,\omega_t)\|^2+2\mathbb{E}\langle g_{\theta}(\theta_{t+1/2},\omega_t),\nabla_{\theta}\mathbb{L}(\theta_t,\omega_t)\rangle.
    \end{aligned}
\end{equation}

For the cross-product term, we divide it as follows:
\begin{equation}
    \begin{aligned}
    &\quad\mathbb{E}\langle g_{\theta}(\theta_{t+1/2},\omega_t),\nabla_{\theta}\mathbb{L}(\theta_t,\omega_t)\rangle\\
    &=\mathbb{E}\langle g_{\theta}(\theta_{t+1/2},\omega_t)-g_{\theta}(\theta_t+\rho\nabla_{\theta}\mathbb{L}(\theta_t,\omega_t),\omega_t),\nabla_{\theta}\mathbb{L}(\theta_t,\omega_t)\rangle+\mathbb{E}\langle g_{\theta}(\theta_t+\rho\nabla_{\theta}\mathbb{L}(\theta_t,\omega_t),\omega_t),\nabla_{\theta}\mathbb{L}(\theta_t,\omega_t)\rangle\\
    &=\mathbb{E}\langle \nabla_{\theta}\mathbb{L}(\theta_{t+1/2},\omega_t)-\nabla_{\theta}\mathbb{L}(\theta_t+\rho\nabla_{\theta}\mathbb{L}(\theta_t,\omega_t),\omega_t),\nabla_{\theta}\mathbb{L}(\theta_t,\omega_t)\rangle+\mathbb{E}\langle\nabla_{\theta}\mathbb{L}(\theta_t+\rho\nabla_{\theta}\mathbb{L}(\theta_t,\omega_t),\omega_t),\nabla_{\theta}\mathbb{L}(\theta_t,\omega_t)\rangle\\
    &\overset{(i)}{\leq}\frac{1}{2}\mathbb{E}\|\nabla_{\theta}\mathbb{L}(\theta_{t+1/2},\omega_t)-\nabla_{\theta}\mathbb{L}(\theta_t+\rho\nabla_{\theta}\mathbb{L}(\theta_t,\omega_t),\omega_t)\|^2+\frac{1}{2}\mathbb{E}\|\nabla_{\theta}\mathbb{L}(\theta_t,\omega_t)\|^2+\mathbb{E}\|\nabla_{\theta}\mathbb{L}(\theta_t,\omega_t)\|^2\\
    &\quad+\mathbb{E}\langle \nabla_{\theta}\mathbb{L}(\theta_t+\rho\nabla_{\theta}\mathbb{L}(\theta_t,\omega_t),\omega_t)-\nabla_{\theta}\mathbb{L}(\theta_t,\omega_t),\nabla_{\theta}\mathbb{L}(\theta_t,\omega_t)\rangle\\
    &\overset{(ii)}{\leq}\frac{\rho^2l^2}{2}\mathbb{E}\|g_{\theta}(\theta_t,\omega_t)-\nabla_{\theta}\mathbb{L}(\theta_t,\omega_t)\|^2+\frac{3}{2}\mathbb{E}\|\nabla_{\theta}\mathbb{L}(\theta_t,\omega_t)\|^2 +\rho l\mathbb{E}\|\nabla_{\theta}\mathbb{L}(\theta_t,\omega_t)\|^2\\
    &\overset{(iii)}{\leq}(\rho l+\frac{3}{2})\mathbb{E}\|\nabla_{\theta}\mathbb{L}(\theta_t,\omega_t)\|^2+\frac{\rho^2l^2\sigma^2}{2M},
    \end{aligned}
\end{equation}
where the inequality $(i)$ is due to the Cauchy-Schwarz inequality; the inequality $(ii)$ is because of the Lipschitz-smoothness of $\mathbb{L}$ that $\mathbb{E}\|\nabla_{\theta}\mathbb{L}(\theta_{t+1/2},\omega_t)-\nabla_{\theta}\mathbb{L}(\theta_t+\rho\nabla_{\theta}\mathbb{L}(\theta_t,\omega_t),\omega_t)\|^2\leq l^2\mathbb{E}\|\theta_{t+1/2}-\theta_t-\rho\nabla_{\theta}\mathbb{L}(\theta_t,\omega_t)\|^2$ and the property of Lipschitz-smoothness that $\langle \nabla_{\theta}\mathbb{L}(\theta_t+\rho\nabla_{\theta}\mathbb{L}(\theta_t,\omega_t),\omega_t)-\nabla_{\theta}\mathbb{L}(\theta_t,\omega_t),\nabla_{\theta}\mathbb{L}(\theta_t,\omega_t)\rangle=\frac{1}{\rho}\langle \nabla_{\theta}\mathbb{L}(\theta_t+\rho\nabla_{\theta}\mathbb{L}(\theta_t,\omega_t),\omega_t)-\nabla_{\theta}\mathbb{L}(\theta_t,\omega_t),\rho\nabla_{\theta}\mathbb{L}(\theta_t,\omega_t)\rangle\leq\frac{l}{\rho}\|\rho\nabla_{\theta}\mathbb{L}(\theta_t,\omega_t)\|^2$; and the inequality $(iii)$ makes use of the Assumption \ref{as:bv}.

As for the second term, we have:
\begin{equation}
    \begin{aligned}
    &\quad\mathbb{E}\|g_{\theta}(\theta_{t+1/2},\omega_t)-\nabla_{\theta}\mathbb{L}(\theta_t,\omega_t)\|^2\\
    &\leq2\mathbb{E}\|g_{\theta}(\theta_{t+1/2},\omega_t)-\nabla_{\theta}\mathbb{L}(\theta_{t+1/2},\omega_t)\|^2+2\mathbb{E}\|\nabla_{\theta}\mathbb{L}(\theta_{t+1/2},\omega_t)-\nabla_{\theta}\mathbb{L}(\theta_t,\omega_t)\|^2\\
    &\leq\frac{2\sigma^2}{M}+2l^2\mathbb{E}\|\theta_{t+1/2}-\theta_t\|^2\\
    &=\frac{2\sigma^2}{M}+2\rho^2l^2\mathbb{E}\|g_{\theta}(\theta_t,\omega_t)\|^2\\
    &\leq2\frac{\sigma^2}{M}(2\rho^2l^2+1)+4\rho^2l^2\mathbb{E}\|\nabla_{\theta}\mathbb{L}(\theta_t,\omega_t)\|^2,
    \end{aligned}
\end{equation}
where the last inequality comes from the fact that: $\mathbb{E}\|g_{\theta}(\theta_t,\omega_t)\|^2\leq2\mathbb{E}\|g_{\theta}(\theta_t,\omega_t)-\nabla_{\theta}\mathbb{L}(\theta_t,\omega_t)\|^2+2\mathbb{E}\|\nabla_{\theta}\mathbb{L}(\theta_t,\omega_t)\|^2$

By combining the above inequalities, we can get:
\begin{equation}
    \mathbb{E}\|g_{\theta}(\theta_{t+1/2},\omega_t)\|^2\leq(4\rho^2l^2+2\rho l+2)\mathbb{E}\|\nabla_{\theta}\mathbb{L}(\theta_t,\omega_t)\|^2+(5\rho^2l^2+2)\frac{\sigma^2}{M}.
\end{equation}
\end{proof}

\begin{lemma}\label{le:L^*gradient}
For the descending relationship of the function $\mathbb{L}^*$, we have:
\begin{equation*}
    \begin{aligned}
    \mathbb{E}[\mathbb{L}^*(\theta_{t+1})]
    &\leq\mathbb{E}[\mathbb{L}^*(\theta_t)]-\frac{\eta_{\theta}}{2}(1-5\rho l-2L\eta_{\theta}(4\rho^2l^2+2\rho l+2))\mathbb{E}\|\nabla\mathbb{L}^*(\theta_t)\|^2\\
    &\quad+[\frac{\eta_{\theta}}{2}(1+\frac{1}{2}\rho l)+L\eta_{\theta}^2(4\rho^2 l^2+2\rho l+2)]\mathbb{E}\|\nabla\mathbb{L}^*(\theta_t)-\nabla_{\theta}\mathbb{L}(\theta_t,\omega_t)\|^2+(5\rho^2l^2+2)\frac{L\eta_{\theta}^2\sigma^2}{2M},
    \end{aligned}
\end{equation*}
where we use the brief notation that $L=l+\frac{l\kappa}{2}$.
\end{lemma}
\begin{proof}
Since $\mathbb{L}^*(\theta)$ is $(l+\frac{l\kappa}{2})$-Lipschitz smooth according to Lemma \ref{le:L*smooth}, we have:
\begin{equation}\label{eq:phi_smooth}
    \begin{aligned}
    \mathbb{L}^*(\theta_{t+1})&\leq \mathbb{L}^*(\theta_t)+\langle\nabla\mathbb{L}^*(\theta_t),\theta_{t+1}-\theta_t\rangle+\frac{1}{2}(l+\frac{l\kappa}{2})\|\theta_{t+1}-\theta_t\|^2\\
    &=\mathbb{L}^*(\theta_t)-\eta_{\theta}\langle\nabla\mathbb{L}^*(\theta_t),g_{\theta}(\theta_{t+1/2},\omega_t)\rangle+\frac{1}{2}(l+\frac{l\kappa}{2})\eta_{\theta}^2\|g_{\theta}(\theta_{t+1/2},\omega_t)\|^2.
    \end{aligned}
\end{equation}

Taking expectation conditioned on $(\theta_t,\omega_t)$ and we get:
\begin{equation}
    \mathbb{E}[\mathbb{L}^*(\theta_{t+1})|\theta_t,\omega_t]\leq\mathbb{L}^*(\theta_t)-\eta_{\theta}\langle\nabla\mathbb{L}^*(\theta_t),\nabla_{\theta}\mathbb{L}(\theta_{t+1/2},\omega_t)\rangle+\frac{1}{2}(l+\frac{l\kappa}{2})\eta_{\theta}^2\mathbb{E}[\|g_{\theta}(\theta_{t+1/2},\omega_t)\|^2|\theta_t,\omega_t].
\end{equation} 

We again take expectation on both side on above inequality so we have:
\begin{equation}\label{eq:stophi_smooth}
    \mathbb{E}[\mathbb{L}^*(\theta_{t+1})]\leq\mathbb{E}[\mathbb{L}^*(\theta_t)]-\eta_{\theta}\mathbb{E}\langle\nabla\mathbb{L}^*(\theta_t),\nabla_{\theta}\mathbb{L}(\theta_{t+1/2},\omega_t)\rangle+\frac{1}{2}(l+\frac{l\kappa}{2})\eta_{\theta}^2\mathbb{E}\|g_{\theta}(\theta_{t+1/2},\omega_t)\|^2.
\end{equation}

For the second term, we decompose it as follows:
\begin{equation}\label{eq:phi_second}
    \begin{aligned}
    &\quad\; \, \mathbb{E}\langle\nabla\mathbb{L}^*(\theta_t),\nabla_{\theta}\mathbb{L}(\theta_{t+1/2},\omega_t)\rangle\\
    &=\mathbb{E}\langle\nabla\mathbb{L}^*(\theta_t),\nabla_{\theta}\mathbb{L}(\theta_t,\omega_t)+\nabla_{\theta}\mathbb{L}(\theta_{t+1/2},\omega_t)-\nabla_{\theta}\mathbb{L}(\theta_t,\omega_t)\rangle\\
    &\geq \mathbb{E}\langle\nabla\mathbb{L}^*(\theta_t),\nabla_{\theta}\mathbb{L}(\theta_t,\omega_t)\rangle-\mathbb{E}\|\nabla\mathbb{L}^*(\theta_t)\|\|\nabla_{\theta}\mathbb{L}(\theta_{t+1/2},\omega_t)-\nabla_{\theta}\mathbb{L}(\theta_t,\omega_t)\|\\
    &\geq \mathbb{E}\langle\nabla\mathbb{L}^*(\theta_t),\nabla_{\theta}\mathbb{L}(\theta_t,\omega_t)\rangle-\rho l\mathbb{E}\|\nabla\mathbb{L}^*(\theta_t)\|\|g_{\theta}(\theta_t,\omega_y)\|\\
    &\geq\mathbb{E}\langle\nabla\mathbb{L}^*(\theta_t),\nabla\mathbb{L}^*(\theta_t)+\nabla_{\theta}\mathbb{L}(\theta_t,\omega_t)-\nabla\mathbb{L}^*(\theta_t)\rangle-\rho l\mathbb{E}\|\nabla\mathbb{L}^*(\theta_t)\|(\|\nabla_{\theta}\mathbb{L}(\theta_t,\omega_t)\|+\|g_{\theta}(\theta_t,\omega_t)-\nabla_{\theta}\mathbb{L}(\theta_t,\omega_t)\|)\\
    &\geq\mathbb{E}\|\nabla\mathbb{L}^*(\theta_t)\|^2-\frac{1}{2}\mathbb{E}\|\nabla\mathbb{L}^*(\theta_t)\|^2-\frac{1}{2}\mathbb{E}\|\nabla_{\theta}\mathbb{L}(\theta_t,\omega_t)-\nabla\mathbb{L}^*(\theta_t)\|^2-\rho l\mathbb{E}\|\nabla\mathbb{L}^*(\theta_t)\|\|\nabla_{\theta}\mathbb{L}(\theta_t,\omega_t)\|\\
    &\quad -\frac{1}{2}\rho l\mathbb{E}\|\nabla\mathbb{L}^*(\theta_t)\|^2-\frac{1}{2}\rho l\mathbb{E}\|g_{\theta}(\theta_t,\omega_t)-\nabla_{\theta}\mathbb{L}(\theta_t,\omega_t)\|^2\\
    &\geq\frac{1-\rho l}{2}\mathbb{E}\|\nabla\mathbb{L}^*(\theta_t)\|^2-\frac{1}{2}\mathbb{E}\|\nabla_{\theta}\mathbb{L}(\theta_t,\omega_t)-\nabla\mathbb{L}^*(\theta_t)\|^2-\rho l\mathbb{E}\|\nabla\mathbb{L}^*(\theta_t)\|\|\nabla_{\theta}\mathbb{L}(\theta_t,\omega_t)\|-\frac{\rho l\sigma^2}{2M}.
    \end{aligned}
\end{equation}

We continue estimating the last term in above inequality \ref{eq:phi_second}
\begin{equation}\label{eq:phi_second_last}
    \begin{aligned}
    &\quad\;\,\mathbb{E}\|\nabla\mathbb{L}^*(\theta_t)\|\|\nabla_{\theta}\mathbb{L}(\theta_t,\omega_t)\|\\
    &=\mathbb{E}\|\nabla\mathbb{L}^*(\theta_t)\|\|\nabla_{\theta}\mathbb{L}(\theta_t,\omega_t)-\nabla\mathbb{L}^*(\theta_t)+\nabla\mathbb{L}^*(\theta_t)\|\\
    &\leq\mathbb{E}\|\nabla\mathbb{L}^*(\theta_t)\|^2+\mathbb{E}\|\nabla\mathbb{L}^*(\theta_t)\|\|\nabla_{\theta}\mathbb{L}(\theta_t,\omega_t)-\nabla\mathbb{L}^*(\theta_t)\|\\
    &\overset{(i)}{\leq}\mathbb{E}\|\nabla\mathbb{L}^*(\theta_t)\|^2+\mathbb{E}\|\nabla\mathbb{L}^*(\theta_t)\|^2+\frac{1}{4}\mathbb{E}\|\nabla_\theta\mathbb{L}(\theta_t,\omega_t)-\nabla\mathbb{L}^*(\theta_t)\|^2,
    \end{aligned}
\end{equation}
where the last inequality $(i)$ is due to Young's inequality.

By combining inequality \ref{eq:stophi_smooth} with \ref{eq:phi_second_last}, we can get:
\begin{equation}\label{eq:phi_second_final}
    \begin{aligned}
    &\quad\; \, \mathbb{E}\langle\nabla\mathbb{L}^*(\theta_t),\nabla_{\theta}\mathbb{L}(\theta_{t+1/2},\omega_t)\rangle\\
    &\geq \frac{1-\rho l}{2}\mathbb{E}\|\nabla\mathbb{L}^*(\theta_t)\|^2-\frac{1}{2}\mathbb{E}\|\nabla_{\theta}\mathbb{L}(\theta_t,\omega_t)-\nabla\mathbb{L}^*(\theta_t)\|^2-2\rho l\mathbb{E}\|\nabla\mathbb{L}^*(\theta_t)\|^2-\frac{\rho l}{4}\mathbb{E}\|\nabla_{\theta}\mathbb{L}(\theta_t,\omega_t)-\nabla\mathbb{L}^*(\theta_t)\|^2-\frac{\rho l\sigma^2}{2M}\\
    &=\frac{1}{2}(1-5\rho l)\mathbb{E}\|\nabla\mathbb{L}^*(\theta_t)\|^2-\frac{1}{2}(1+\frac{1}{2}\rho l)\mathbb{E}\|\nabla\mathbb{L}^*(\theta_t)-\nabla_{\theta}\mathbb{L}(\theta_t,\omega_t)\|^2-\frac{\rho l\sigma^2}{2M}.
    \end{aligned}
\end{equation}

Finally, we combine inequality \ref{eq:stophi_smooth} with Lemma \ref{eq:gradientbound} and inequality \ref{eq:phi_second_final}:
\begin{equation}
    \begin{aligned}
    &\quad\;\,\mathbb{E}[\mathbb{L}^*(\theta_{t+1})]\\
    &\leq\mathbb{E}[\mathbb{L}^*(\theta_t)]-\frac{\eta_{\theta}}{2}(1-5\rho l)\mathbb{E}\|\nabla\mathbb{L}^*(\theta_t)\|^2+\frac{\eta_{\theta}}{2}(1+\frac{1}{2}\rho l)\mathbb{E}\|\nabla\mathbb{L}^*(\theta_t)-\nabla_{\theta}\mathbb{L}(\theta_t,\omega_t)\|^2\\
    &\quad+\frac{1}{2}(l+\frac{l\kappa}{2})\eta_{\theta}^2((4\rho^2l^2+2\rho l+2)\mathbb{E}\|\nabla_{\theta}\mathbb{L}(\theta_t,\omega_t)\|^2+(5\rho^2l^2+2)\frac{\sigma^2}{M})\\
    &\overset{(i)}{\leq}\mathbb{E}[\mathbb{L}^*(\theta_t)]-\frac{\eta_{\theta}}{2}(1-5\rho l-\eta_{\theta}(2l+l\kappa)(4\rho^2l^2+2\rho l+2))\mathbb{E}\|\nabla\mathbb{L}^*(\theta_t)\|^2\\
    &\quad+[\frac{\eta_{\theta}}{2}(1+\frac{1}{2}\rho l)+\eta_{\theta}^2(l+\frac{l\kappa}{2})(4\rho^2 l^2+2\rho l+2)]\mathbb{E}\|\nabla\mathbb{L}^*(\theta_t)-\nabla_{\theta}\mathbb{L}(\theta_t,\omega_t)\|^2+\frac{1}{2}(l+\frac{l\kappa}{2})(5\rho^2l^2+2)\frac{\eta_{\theta}^2\sigma^2}{M},
    \end{aligned}
\end{equation}
where the last inequality $(i)$ uses the Cauchy-Schwarz inequality that $\|\nabla_{\theta}\mathbb{L}(\theta_t,\omega_t)\|^2\leq2\|\nabla\mathbb{L}^*(\theta_t)\|^2+2\|\nabla_{\theta}\mathbb{L}(\theta_t,\omega_t)-\nabla\mathbb{L}^*(\theta_t)\|^2$.
\end{proof}

\subsection{Theorem}
\begin{theorem}\label{le:pl:phi}
Under Assumption \ref{as:bv},\ref{as:smooth},\ref{as:pl}, and the learning rate satisfy that $\eta_{\theta}\leq\min\{\frac{1}{128\kappa^2l},\sqrt{\frac{M(\mathbb{E}[\mathbb{L}^*(\theta_0)]-\min_{\theta}\mathbb{E}[\mathbb{L}^*(\theta)])}{132T\kappa^4l\sigma^2}}\}$, $\eta_{\omega}\leq64\kappa^2\eta_{\theta}$ and $\rho\leq\frac{\eta_{\theta}}{2l}$, we have the convergence bound for our problem:
\begin{equation}
    \frac{1}{T}\sum_{t=0}^{T-1}\mathbb{E}\|\nabla\mathbb{L}^*(\theta_t)\|^2\leq320\sqrt{\frac{3\kappa^4l(\mathbb{E}[\mathbb{L}^*(\theta_0)]-\min_{\theta}\mathbb{E}[\mathbb{L}^*(\theta)])\sigma^2}{11MT}}=\mathcal{O}(\frac{\kappa^2}{\sqrt{MT}}).
\end{equation}

\end{theorem}

\begin{proof}
First recall the descending relationship of the function $\mathbb{L}^*$ in Lemma \ref{le:L^*gradient}:
\begin{equation}
    \begin{aligned}
    &\quad\;\,\mathbb{E}[\mathbb{L}^*(\theta_{t+1})]\\
    &\leq\mathbb{E}[\mathbb{L}^*(\theta_t)]-\frac{\eta_{\theta}}{2}(1-5\rho l-2L\eta_{\theta}(4\rho^2l^2+2\rho l+2))\mathbb{E}\|\nabla\mathbb{L}^*(\theta_t)\|^2\\
    &\quad+[\frac{\eta_{\theta}}{2}(1+\frac{1}{2}\rho l)+L\eta_{\theta}^2(4\rho^2 l^2+2\rho l+2)]\mathbb{E}\|\nabla\mathbb{L}^*(\theta_t)-\nabla_{\theta}\mathbb{L}(\theta_t,\omega_t)\|^2+(5\rho^2l^2+2)\frac{L\eta_{\theta}^2\sigma^2}{2M}.
    \end{aligned}
\end{equation}

Then, using the smoothness of the variables $\theta$ and $\omega$ respectively, we can get:
\begin{equation*}
    \begin{aligned}
    \mathbb{L}(\theta_{t+1},\omega_t)&\geq \mathbb{L}(\theta_t,\omega_t)+\langle\nabla_{\theta}\mathbb{L}(\theta_t,\omega_t),\theta_{t+1}-\theta_t\rangle-\frac{l}{2}\|\theta_{t+1}-\theta_t\|^2;    \\
    \mathbb{L}(\theta_{t+1},\omega_{t+1})&\geq \mathbb{L}(\theta_{t+1},\omega_t)+\langle\nabla_{\omega}\mathbb{L}(\theta_{t+1},\omega_t),\omega_{t+1}-\omega_t\rangle-\frac{l}{2}\|\omega_{t+1}-\omega_t\|^2.
    \end{aligned}
\end{equation*}

Taking expectation we can get:
\begin{equation}
    \begin{aligned}
    \mathbb{E}[\mathbb{L}(\theta_{t+1},\omega_t)]&\geq\mathbb{E}[\mathbb{L}(\theta_t,\omega_t)]-\eta_{\theta}\mathbb{E}\langle\nabla_{\theta}\mathbb{L}(\theta_t,\omega_t),\nabla_{\theta}\mathbb{L}(\theta_{t+1/2},\omega_t)\rangle-\frac{l\eta_{\theta}^2}{2}\mathbb{E}\|g_{\theta}(\theta_{t+1/2},\omega_t)\|^2\\
    &\geq\mathbb{E}[\mathbb{L}(\theta_t,\omega_t)]-\eta_{\theta}\mathbb{E}\|\nabla_{\theta}\mathbb{L}(\theta_t,\omega_t)\|^2-\frac{\eta_{\theta}}{2}\mathbb{E}\|\nabla_{\theta}\mathbb{L}(\theta_t,\omega_t)\|^2\\&\quad-\frac{\eta_{\theta}}{2}\mathbb{E}\|\nabla_{\theta}\mathbb{L}(\theta_{t+1/2},\omega_t)-\nabla_{\theta}\mathbb{L}(\theta_t,\omega_t)\|^2-\frac{l\eta_{\theta}^2}{2}\mathbb{E}\|g_{\theta}(\theta_{t+1/2},\omega_t)\|^2\\
    &\geq \mathbb{E}[\mathbb{L}(\theta_t,\omega_t)]-\frac{3\eta_{\theta}}{2}\mathbb{E}\|\nabla_{\theta}\mathbb{L}(\theta_t,\omega_t)\|^2-\frac{l^2\rho^2\eta_{\theta}}{2}\mathbb{E}\|g_{\theta}(\theta_t,\omega_t)\|^2-\frac{l\eta_{\theta}^2}{2}\mathbb{E}\|g_{\theta}(\theta_{t+1/2},\omega_t)\|^2\\
    &\geq\mathbb{E}[\mathbb{L}(\theta_t,\omega_t)]-(\frac{3\eta_{\theta}}{2}+\frac{l^2\rho^2\eta_{\theta}}{2}+l\eta_{\theta}^2(2\rho^2l^2+\rho l+1))\mathbb{E}\|\nabla_{\theta}\mathbb{L}(\theta_t,\omega_t)\|^2\\&\quad-(\frac{l^2\rho^2\eta_{\theta}}{2}+\frac{l\eta_{\theta}^2}{2}(5\rho^2l^2+2))\frac{\sigma^2}{M};\\
    \mathbb{E}[\mathbb{L}(\theta_{t+1},\omega_{t+1})]&\geq\mathbb{E}[\mathbb{L}(\theta_{t+1},\omega_t)]+\eta_{\omega}\mathbb{E}\langle\nabla_{\omega}\mathbb{L}(\theta_{t+1},\omega_t),\nabla_{\omega}\mathbb{L}(\theta_t,\omega_t)\rangle-\frac{l\eta_{\omega}^2}{2}\mathbb{E}\|g_{\omega}(\theta_t,\omega_t)\|^2\\
    &\geq\mathbb{E}[\mathbb{L}(\theta_{t+1},\omega_t)]+\frac{\eta_{\omega}}{2}\mathbb{E}\|\nabla_{\omega}\mathbb{L}(\theta_t,\omega_t)\|^2-\frac{\eta_{\omega}}{2}\mathbb{E}\|\nabla_{\omega}\mathbb{L}(\theta_{t+1},\omega_t)-\nabla_{\omega}\mathbb{L}(\theta_t,\omega_t)\|^2-\frac{l\eta_{\omega}^2}{2}\mathbb{E}\|g_{\omega}(\theta_t,\omega_t)\|^2\\
    &\geq\mathbb{E}[\mathbb{L}(\theta_{t+1},\omega_t)]+\frac{\eta_{\omega}}{2}\mathbb{E}\|\nabla_{\omega}\mathbb{L}(\theta_t,\omega_t)\|^2-\frac{l\eta_{\theta}^2\eta_{\omega}}{2}\mathbb{E}\|g_{\theta}(\theta_{t+1/2},\omega_t)\|^2-\frac{l\eta_{\omega}^2}{2}\mathbb{E}\|g_{\omega}(\theta_t,\omega_t)\|^2\\
    &\geq\mathbb{E}[\mathbb{L}(\theta_{t+1},\omega_t)]+(\frac{\eta_{\omega}}{2}-\frac{l\eta_{\omega}^2}{2})\mathbb{E}\|\nabla_{\omega}\mathbb{L}(\theta_t,\omega_t)\|^2-(l\eta_{\theta}^2\eta_{\omega}(2\rho^2l^2+\rho l+1))\mathbb{E}\|\nabla_{\theta}\mathbb{L}(\theta_t,\omega_t)\|^2\\
    &\quad-(\frac{l\eta_{\omega}^2}{2}+\frac{l\eta_{\theta}^2\eta_{\omega}}{2}(5\rho^2l^2+2))\frac{\sigma^2}{M}.
    \end{aligned}
\end{equation}

Then we construct a potential function in the same way as~\cite{yang2022faster}:
$$V_t=V(\theta_t,\omega_t)=\mathbb{L}^*(\theta_t)+\alpha[\mathbb{L}^*(\theta_t)-\mathbb{L}(\theta_t,\omega_t)],$$
where $\alpha>0$ is a preset parameter. Then we come to evaluate the descending relationship of the potential function $V_t$.

Combining the above inequalities we can get the descending relationship of the potential function:
\begin{align}
    &\quad \mathbb{E}[V_{t+1}]-\mathbb{E}[V_t]  \notag\\
    &=(1+\alpha)(\mathbb{E}[\mathbb{L}^*(\theta_{t+1})]-\mathbb{E}[\mathbb{L}^*(\theta_t)])-\alpha(\mathbb{E}[\mathbb{L}(\theta_{t+1},\omega_{t+1})]-\mathbb{E}[\mathbb{L}(\theta_t,\omega_t)])  \notag\\
    &\leq(1+\alpha)\{-\frac{\eta_{\theta}}{2}(1-5\rho l-2L\eta_{\theta}(4\rho^2l^2+2\rho l+2))\mathbb{E}\|\nabla\mathbb{L}^*(\theta_t)\|^2  \notag\\
    &\quad+[\frac{\eta_{\theta}}{2}(1+\frac{1}{2}\rho l)+L\eta_{\theta}^2(4\rho^2 l^2+2\rho l+2)]\mathbb{E}\|\nabla\mathbb{L}^*(\theta_t)-\nabla_{\theta}\mathbb{L}(\theta_t,\omega_t)\|^2+(5\rho^2l^2+2)\frac{L\eta_{\theta}^2\sigma^2}{2M}\}  \notag\\
    &\quad-\alpha\{-(\frac{3\eta_{\theta}}{2}+\frac{l^2\rho^2\eta_{\theta}}{2}+l\eta_{\theta}^2(2\rho^2l^2+\rho l+1))\mathbb{E}\|\nabla_{\theta}\mathbb{L}(\theta_t,\omega_t)\|^2-(\frac{l^2\rho^2\eta_{\theta}}{2}+\frac{l\eta_{\theta}^2}{2}(5\rho^2l^2+2))\frac{\sigma^2}{M}  \notag\\
    &\quad+(\frac{\eta_{\omega}}{2}-\frac{l\eta_{\omega}^2}{2})\mathbb{E}\|\nabla_{\omega}\mathbb{L}(\theta_t,\omega_t)\|^2-(l\eta_{\theta}^2\eta_{\omega}(2\rho^2l^2+\rho l+1))\mathbb{E}\|\nabla_{\theta}\mathbb{L}(\theta_t,\omega_t)\|^2-(\frac{l\eta_{\omega}^2}{2}+\frac{l\eta_{\theta}^2\eta_{\omega}}{2}(5\rho^2l^2+2))\frac{\sigma^2}{M}\}  \tag{\stepcounter{equation}\theequation}\\
    &=-\frac{\eta_{\theta}}{2}(1+\alpha)(1-5\rho l-2L\eta_{\theta}(4\rho^2l^2+2\rho l+2))\mathbb{E}\|\nabla\mathbb{L}^*(\theta_t)\|^2  \notag\\
    &\quad+(1+\alpha)(\frac{\eta_{\theta}}{2}(1+\frac{1}{2}\rho l)+L\eta_{\theta}^2(4\rho^2 l^2+2\rho l+2))\mathbb{E}\|\nabla\mathbb{L}^*(\theta_t)-\nabla_{\theta}\mathbb{L}(\theta_t,\omega_t)\|^2  \notag\\
    &\quad+\alpha[(\frac{3\eta_{\theta}}{2}+\frac{l^2\rho^2\eta_{\theta}}{2}+l\eta_{\theta}^2(2\rho^2l^2+\rho l+1))+l\eta_{\theta}^2\eta_{\omega}(2\rho^2l^2+\rho l+1)]\mathbb{E}\|\nabla_{\theta}\mathbb{L}(\theta_t,\omega_t)\|^2  \notag\\
    &\quad-\alpha(\frac{\eta_{\omega}}{2}-\frac{l\eta_{\omega}^2}{2})\mathbb{E}\|\nabla_{\omega}\mathbb{L}(\theta_t,\omega_t)\|^2  \notag\\
    &\quad+[(1+\alpha)(5\rho^2l^2+2)\frac{L\eta_{\theta}^2}{2}+\alpha(\frac{l^2\rho^2\eta_{\theta}}{2}+\frac{l\eta_{\theta}^2}{2}(5\rho^2l^2+2))+\alpha(\frac{l\eta_{\omega}^2}{2}+\frac{l\eta_{\theta}^2\eta_{\omega}}{2}(5\rho^2l^2+2))]\frac{\sigma^2}{M}  \notag\\
    &\leq-\{\frac{\eta_{\theta}}{2}(1+\alpha)(1-5\rho l-2L\eta_{\theta}(4\rho^2l^2+2\rho l+2))-2\alpha[(\frac{3\eta_{\theta}}{2}+\frac{l^2\rho^2\eta_{\theta}}{2}+l\eta_{\theta}^2(2\rho^2l^2+\rho l+1))+l\eta_{\theta}^2\eta_{\omega}(2\rho^2l^2+\rho l+1)]\}\notag\\&\quad\mathbb{E}\|\nabla\mathbb{L}^*(\theta_t)\|^2  \notag\\
    &\quad+\{(1+\alpha)(\frac{\eta_{\theta}}{2}(1+\frac{1}{2}\rho l)+L\eta_{\theta}^2(4\rho^2 l^2+2\rho l+2))+2\alpha[(\frac{3\eta_{\theta}}{2}+\frac{l^2\rho^2\eta_{\theta}}{2}+l\eta_{\theta}^2(2\rho^2l^2+\rho l+1))+l\eta_{\theta}^2\eta_{\omega}(2\rho^2l^2+\rho l+1)]\}  \notag\\
    &\quad\mathbb{E}\|\nabla\mathbb{L}^*(\theta_t)-\nabla_{\theta}\mathbb{L}(\theta_t,\omega_t)\|^2  \notag\\
    &\quad-\alpha(\frac{\eta_{\omega}}{2}-\frac{l\eta_{\omega}^2}{2})\mathbb{E}\|\nabla_{\omega}\mathbb{L}(\theta_t,\omega_t)\|^2  \notag\\
    &\quad+[(1+\alpha)(5\rho^2l^2+2)\frac{L\eta_{\theta}^2}{2}+\alpha(\frac{l^2\rho^2\eta_{\theta}}{2}+\frac{l\eta_{\theta}^2}{2}(5\rho^2l^2+2))+\alpha(\frac{l\eta_{\omega}^2}{2}+\frac{l\eta_{\theta}^2\eta_{\omega}}{2}(5\rho^2l^2+2))]\frac{\sigma^2}{M}.
\end{align}

Since we have the following property according to Lemma \ref{le:L*smooth} and the PL condition \ref{as:pl}:
$$\|\nabla\mathbb{L}^*(\theta_t)-\nabla_{\theta}f(\theta_t,\omega_t)\|\leq l\|\omega^*(\theta_t)-\omega_t\|\leq\kappa\|\nabla_{\omega}f(\theta_t,\omega_t)\|.$$

So we can further the above inequality as follows:
\begin{equation}
    \begin{aligned}
    &\quad\mathbb{E}[V_{t+1}]-\mathbb{E}[V_t]\\
    &\leq-\{\frac{\eta_{\theta}}{2}(1+\alpha)(1-5\rho l-2L\eta_{\theta}(4\rho^2l^2+2\rho l+2))-2\alpha[(\frac{3\eta_{\theta}}{2}+\frac{l^2\rho^2\eta_{\theta}}{2}+l\eta_{\theta}^2(2\rho^2l^2+\rho l+1))+l\eta_{\theta}^2\eta_{\omega}(2\rho^2l^2+\rho l+1)]\}\\&\quad\mathbb{E}\|\nabla\mathbb{L}^*(\theta_t)\|^2\\
    &\quad-\{\alpha(\frac{\eta_{\omega}}{2}-\frac{l\eta_{\omega}^2}{2})-\kappa^2[(1+\alpha)(\frac{\eta_{\theta}}{2}(1+\frac{1}{2}\rho l)+L\eta_{\theta}^2(4\rho^2 l^2+2\rho l+2))\\&\quad+2\alpha[(\frac{3\eta_{\theta}}{2}+\frac{l^2\rho^2\eta_{\theta}}{2}+l\eta_{\theta}^2(2\rho^2l^2+\rho l+1))+l\eta_{\theta}^2\eta_{\omega}(2\rho^2l^2+\rho l+1)]]\}\mathbb{E}\|\nabla_{\omega}\mathbb{L}(\theta_t,\omega_t)\|^2\\
    &\quad+[(1+\alpha)(5\rho^2l^2+2)\frac{L\eta_{\theta}^2}{2}+\alpha(\frac{l^2\rho^2\eta_{\theta}}{2}+\frac{l\eta_{\theta}^2}{2}(5\rho^2l^2+2))+\alpha(\frac{l\eta_{\omega}^2}{2}+\frac{l\eta_{\theta}^2\eta_{\omega}}{2}(5\rho^2l^2+2))]\frac{\sigma^2}{M}.
    \end{aligned}
\end{equation}

Then we require the parameters satisfy: $\alpha=\frac{1}{16}$, $\rho l \leq\frac{1}{16}$, $\eta_{\theta}(2\rho l+1)^2\kappa l\leq\frac{1}{64}$, $\kappa^2\eta_{\theta}l\leq\frac{1}{128}$, $\rho\leq\frac{\eta_{\theta}}{2l}$ and $\eta_{\omega}\leq64\kappa^2\eta_{\theta}$.

So the inequality can be further simplified as:
\begin{equation}
    \begin{aligned}
    &\quad\mathbb{E}[V_{t+1}]-\mathbb{E}[V_t]\\
    &\leq-\frac{11}{80}\eta_{\theta}\mathbb{E}\|\nabla\mathbb{L}^*(\theta_t)\|^2-\frac{41}{32}\eta_{\theta}\kappa^2\mathbb{E}\|\nabla_{\omega}f(\theta_t,\omega_t)\|^2+129\kappa^4l\eta_{\theta}^2\frac{\sigma^2}{M}.
    \end{aligned}
\end{equation}

Telescoping the above inequality we can get:
\begin{equation}\label{eq:bound_1}
    \begin{aligned}
    \frac{1}{T}\sum_{t=0}^{T-1}\mathbb{E}\|\nabla\mathbb{L}^*(\theta_t)\|^2\leq\frac{80}{11\eta_{\theta}T}(\mathbb{E}[V_0]-\mathbb{E}[V_T])+960\kappa^4l\eta_{\theta}\frac{\sigma^2}{M}.
    \end{aligned}
\end{equation}

Further, we can evaluate the first term that:
\begin{equation*}
    \begin{aligned}
    \mathbb{E}[V_0]-\mathbb{E}[V_T]&\leq\mathbb{E}[V_0]-\min_{\theta,\omega}\mathbb{E}[V(\theta,\omega)]\\
    &\leq \mathbb{E}[\mathbb{L}^*(\theta_0)]-\min_{\theta}\mathbb{E}[\mathbb{L}^*(\theta)]+\frac{1}{16}(\mathbb{E}[\mathbb{L}^*(\theta_0)]-\mathbb{E}[\mathbb{L}(\theta_0,\omega_0)])\\
    &=\mathbb{E}[\mathbb{L}^*(\theta_0)]-\min_{\theta}\mathbb{E}[\mathbb{L}^*(\theta)]+\frac{1}{16}\Delta_0,
    \end{aligned}
\end{equation*}
where we denote the initial error as: $\Delta_0=\mathbb{E}[\mathbb{L}^*(\theta_0)]-\mathbb{E}[\mathbb{L}(\theta_0,\omega_0)].$

Therefore, the inequality \ref{eq:bound_1} can be further evaluated as:
\begin{equation}\label{eq:bound_2}
    \begin{aligned}
    \frac{1}{T}\sum_{t=0}^{T-1}\mathbb{E}\|\nabla\mathbb{L}^*(\theta_t)\|^2\leq\frac{80}{11\eta_{\theta}T}(\mathbb{E}[\mathbb{L}^*(\theta_0)]-\min_{\theta}\mathbb{E}[\mathbb{L}^*(\theta)])+\frac{5}{11\eta_{\theta}T}\Delta_0+960\kappa^4l\eta_{\theta}\frac{\sigma^2}{M},
    \end{aligned}
\end{equation}
when we select $\eta_{\theta}=\sqrt{\frac{M(\mathbb{E}[\mathbb{L}^*(\theta_0)]-\min_{\theta}\mathbb{E}[\mathbb{L}^*(\theta)])}{132T\kappa^4l\sigma^2}}$, and samples can be minibatch, the convergence can be bounded by:
\begin{equation}
    \frac{1}{T}\sum_{t=0}^{T-1}\mathbb{E}\|\nabla\mathbb{L}^*(\theta_t)\|^2\leq320\sqrt{\frac{3\kappa^4l(\mathbb{E}[\mathbb{L}^*(\theta_0)]-\min_{\theta}\mathbb{E}[\mathbb{L}^*(\theta)])\sigma^2}{11MT}}=\mathcal{O}(\frac{\kappa^2}{\sqrt{MT}}).
\end{equation}

\end{proof}

\begin{table*}
	\small
	\centering
	\caption{\small Quantitative comparisons on distribution-aware robust generalization setting. Averaged accuracy ($\%$) with standard deviations are computed over three independent trails.}
	\vspace{-3mm}
	\setlength{\tabcolsep}{1.8mm}
	\label{tab:appendix_distribution_aware}
	\begin{tabular}{lllcccccc}
		\toprule[1pt]
		\multirow{2}{*}{Dataset} & \multirow{2}{*}{Type} & \multirow{2}{*}{Method} & \multicolumn{6}{c}{Corruption Severity} \\
		&  &  & \multicolumn{1}{c}{0} & \multicolumn{1}{c}{1} & \multicolumn{1}{c}{2} & \multicolumn{1}{c}{3} & \multicolumn{1}{c}{4} & \multicolumn{1}{c}{5} \\ \midrule[0.6pt]
		\multirow{10}{*}{CIFAR10} & \multirow{5}{*}{Snow} & ERM & \multicolumn{1}{c}{$90.8\pm0.01$} & \multicolumn{1}{c}{$90.1\pm0.02$} & \multicolumn{1}{c}{$88.1\pm0.02$} & \multicolumn{1}{c}{$88.1\pm0.02$} & \multicolumn{1}{c}{$85.7\pm0.02$} & \multicolumn{1}{c}{$82.6\pm0.01$} \\
		&  & IRM & \multicolumn{1}{c}{$91.1\pm0.02$} & \multicolumn{1}{c}{$90.7\pm0.01$} & \multicolumn{1}{c}{$89.7\pm0.02$} & \multicolumn{1}{c}{$88.0\pm0.03$} & \multicolumn{1}{c}{$84.6\pm0.02$} & \multicolumn{1}{c}{$83.2\pm0.03$} \\
		&  & REx & \multicolumn{1}{c}{$91.8\pm0.02$} & \multicolumn{1}{c}{$\bm{91.9}\pm\bm{0.01}$} & \multicolumn{1}{c}{$88.4\pm0.01$} & \multicolumn{1}{c}{$88.3\pm0.01$} & \multicolumn{1}{c}{$88.6\pm0.01$} & \multicolumn{1}{c}{$83.0\pm0.02$} \\
		&  & GroupDRO & \multicolumn{1}{c}{$91.5\pm0.02$} & \multicolumn{1}{c}{$91.0\pm0.01$} & \multicolumn{1}{c}{$88.7\pm0.02$} & \multicolumn{1}{c}{$88.6\pm0.02$} & \multicolumn{1}{c}{$85.2\pm0.03$} & \multicolumn{1}{c}{$83.5\pm0.02$} \\
		&  & SharpDRO & \multicolumn{1}{c}{$\bm{93.1}\pm\bm{0.01}$} & \multicolumn{1}{c}{$91.8\pm0.01$} & \multicolumn{1}{c}{$\bm{90.5}\pm\bm{0.02}$} & \multicolumn{1}{c}{$\bm{90.8}\pm\bm{0.02}$} & \multicolumn{1}{c}{$\bm{87.9}\pm\bm{0.01}$} & \multicolumn{1}{c}{$\bm{84.3}\pm\bm{0.02}$} \\ \cline{2-9} 
		
		& \multirow{5}{*}{Shot} & ERM & \multicolumn{1}{c}{$\bm{92.5}\pm\bm{0.02}$} & \multicolumn{1}{c}{$91.1\pm0.02$} & \multicolumn{1}{c}{$89.9\pm0.01$} & \multicolumn{1}{c}{$85.6\pm0.03$} & \multicolumn{1}{c}{$85.7\pm0.01$} & \multicolumn{1}{c}{$78.8\pm0.01$} \\
		&  & IRM & \multicolumn{1}{c}{$90.4\pm0.01$} & \multicolumn{1}{c}{$90.3\pm0.02$} & \multicolumn{1}{c}{$89.4\pm0.02$} & \multicolumn{1}{c}{$86.3\pm0.01$} & \multicolumn{1}{c}{$84.3\pm0.02$} & \multicolumn{1}{c}{$79.1\pm0.02$} \\
		&  & REx & \multicolumn{1}{c}{$91.1\pm0.02$} & \multicolumn{1}{c}{$90.6\pm0.02$} & \multicolumn{1}{c}{$90.2\pm0.03$} & \multicolumn{1}{c}{$86.8\pm0.02$} & \multicolumn{1}{c}{$84.7\pm0.02$} & \multicolumn{1}{c}{$80.5\pm0.01$} \\
		&  & GroupDRO & \multicolumn{1}{c}{$92.2\pm0.01$} & \multicolumn{1}{c}{$\bm{91.4}\pm\bm{0.01}$} & \multicolumn{1}{c}{$89.4\pm0.02$} & \multicolumn{1}{c}{$84.0\pm0.01$} & \multicolumn{1}{c}{$84.7\pm0.02$} & \multicolumn{1}{c}{$78.3\pm0.01$} \\
		&  & SharpDRO & \multicolumn{1}{c}{$92.4\pm0.02$} & \multicolumn{1}{c}{$91.1\pm0.02$} & \multicolumn{1}{c}{$\bm{90.3}\pm\bm{0.02}$} & \multicolumn{1}{c}{$\bm{87.5}\pm\bm{0.02}$} & \multicolumn{1}{c}{$\bm{86.4}\pm\bm{0.02}$} & \multicolumn{1}{c}{$\bm{83.3}\pm\bm{0.02}$} \\
		\midrule[0.6pt]

		\multirow{10}{*}{CIFAR100} & \multirow{5}{*}{Snow} & ERM & $67.7\pm0.01$ & $68.1\pm0.01$ & $64.7\pm0.01$ & $63.1\pm0.01$ & $60.5\pm0.02$ & $57.3\pm0.01$ \\
		&  & IRM & $69.3\pm0.01$ & $67.5\pm0.02$ & $64.9\pm0.02$ & $61.0\pm0.01$ & $58.2\pm0.01$ & $55.1\pm0.01$ \\
		&  & REx & $66.4\pm0.01$ & $65.9\pm0.01$ & $62.4\pm0.01$ & $61.2\pm0.02$ & $57.5\pm0.03$ & $56.0\pm0.02$ \\
		&  & GroupDRO & $68.0\pm0.02$ & $68.2\pm0.01$ & $65.1\pm0.01$ & $60.9\pm0.03$ & $59.8\pm0.01$ & $58.1\pm0.02$ \\
		&  & SharpDRO & $\bm{71.5}\pm\bm{0.01}$ & $\bm{70.8}\pm\bm{0.03}$ & $\bm{67.5}\pm\bm{0.02}$ & $\bm{65.5}\pm\bm{0.01}$ & $\bm{62.3}\pm\bm{0.01}$ & $\bm{59.2}\pm\bm{0.03}$ \\ \cline{2-9}

		& \multirow{5}{*}{Shot} & ERM & $67.6\pm0.03$ & $65.1\pm0.01$ & $62.9\pm0.01$ & $56.0\pm0.01$ & $55.1\pm0.01$ & $47.3\pm0.01$ \\
		&  & IRM & $67.5\pm0.02$ & $65.7\pm0.01$ & $62.7\pm0.01$ & $59.5\pm0.01$ & $55.8\pm0.01$ & $48.3\pm0.01$ \\
		&  & REx & $65.7\pm0.01$ & $63.8\pm0.02$ & $61.9\pm0.01$ & $59.3\pm0.03$ & $53.8\pm0.01$ & $48.1\pm0.01$ \\
		&  & GroupDRO & $67.0\pm0.02$ & $65.8\pm0.01$ & $63.1\pm0.01$ & $58.9\pm0.01$ & $57.5\pm0.01$ & $49.3\pm0.01$ \\
		&  & SharpDRO & $\bm{69.2}\pm\bm{0.01}$ & $\bm{67.3}\pm\bm{0.02}$ & $\bm{65.4}\pm\bm{0.03}$ & $\bm{62.5}\pm\bm{0.01}$ & $\bm{57.7}\pm\bm{0.02}$ & $\bm{51.6}\pm\bm{0.01}$ \\ \midrule[0.6pt]

		\multirow{10}{*}{ImageNet30} & \multirow{5}{*}{Snow} & ERM & $86.7\pm0.03$ & $85.2\pm0.01$ & $83.4\pm0.01$ & $81.1\pm0.01$ & $75.3\pm0.01$ & $75.6\pm0.01$ \\
		&  & IRM & $85.6\pm0.01$ & $84.0\pm0.02$ & $82.1\pm0.03$ & $79.7\pm0.01$ & $75.0\pm0.01$ & $75.6\pm0.01$ \\
		&  & REx & $85.4\pm0.01$ & $84.6\pm0.02$ & $82.7\pm0.02$ & $80.5\pm0.03$ & $75.7\pm0.03$ & $75.9\pm0.03$ \\
		&  & GroupDRO & $86.7\pm0.01$ & $85.5\pm0.03$ & $83.4\pm0.01$ & $81.2\pm0.02$ & $76.3\pm0.01$ & $76.6\pm0.01$ \\
		&  & SharpDRO & $\bm{88.2}\pm\bm{0.02}$ & $\bm{88.2}\pm\bm{0.01}$ & $\bm{85.4}\pm\bm{0.02}$ & $\bm{81.9}\pm\bm{0.01}$ & $\bm{79.8}\pm\bm{0.03}$ & $\bm{79.5}\pm\bm{0.02}$ \\ \cline{2-9}

		& \multirow{5}{*}{Shot} & ERM & $86.9\pm0.01$ & $84.8\pm0.01$ & $83.6\pm0.01$ & $79.7\pm0.01$ & $75.4\pm0.01$ & $64.6\pm0.01$ \\
		&  & IRM & $86.8\pm0.01$ & $85.1\pm0.03$ & $81.5\pm0.01$ & $73.5\pm0.02$ & $68.5\pm0.03$ & $62.5\pm0.03$ \\
		&  & REx & $83.8\pm0.01$ & $86.3\pm0.03$ & $82.5\pm0.02$ & $73.9\pm0.01$ & $70.6\pm0.03$ & $64.0\pm0.02$ \\
		&  & GroupDRO & $86.7\pm0.01$ & $85.6\pm0.03$ & $84.5\pm0.01$ & $80.7\pm0.01$ & $76.2\pm0.04$ & $65.4\pm0.01$ \\
		&  & SharpDRO & $\bm{88.1}\pm\bm{0.01}$ & $\bm{87.2}\pm\bm{0.02}$ & $\bm{84.7}\pm\bm{0.01}$ & $\bm{82.2}\pm\bm{0.01}$ & $\bm{78.2}\pm\bm{0.01}$ & $\bm{67.9}\pm\bm{0.02}$ \\  \bottomrule[1pt]
	\end{tabular}
	\vspace{-2mm}
\end{table*}

\begin{table*}[t]
	\small
	\centering
	\caption{\small Quantitative comparisons on distribution-agnostic robust generalization setting. Averaged accuracy ($\%$) with standard deviations are computed over three independent trails.}
	\vspace{-3mm}
	\setlength{\tabcolsep}{1.8mm}
	\label{tab:appendix_distribution_agnostic}
	\begin{tabular}{lllcccccc}
		\toprule[1pt]
		\multirow{2}{*}{Dataset} & \multirow{2}{*}{Type} & \multirow{2}{*}{Method} & \multicolumn{6}{c}{Corruption Severity} \\
		&  &  & \multicolumn{1}{c}{0} & \multicolumn{1}{c}{1} & \multicolumn{1}{c}{2} & \multicolumn{1}{c}{3} & \multicolumn{1}{c}{4} & \multicolumn{1}{c}{5} \\ \midrule[0.6pt]
		\multirow{6}{*}{CIFAR10} & \multirow{3}{*}{Snow} & JTT & \multicolumn{1}{c}{$88.6\pm0.02$} & \multicolumn{1}{c}{$87.8\pm0.03$} & \multicolumn{1}{c}{$86.5\pm0.02$} & \multicolumn{1}{c}{$87.2\pm0.02$} & \multicolumn{1}{c}{$84.2\pm0.02$} & \multicolumn{1}{c}{$83.2\pm0.03$} \\
		&  & EIIL & \multicolumn{1}{c}{$88.3\pm0.02$} & \multicolumn{1}{c}{$87.8\pm0.01$} & \multicolumn{1}{c}{$85.6\pm0.02$} & \multicolumn{1}{c}{$87.3\pm0.03$} & \multicolumn{1}{c}{$85.2\pm0.04$} & \multicolumn{1}{c}{$82.3\pm0.01$} \\
		&  & SharpDRO & $\bm{91.6}\pm\bm{0.01}$ & $\bm{91.1}\pm\bm{0.02}$ & $\bm{90.8}\pm\bm{0.01}$ & $\bm{89.7}\pm\bm{0.02}$ & $\bm{86.2}\pm\bm{0.01}$ & $\bm{83.8}\pm\bm{0.02}$ \\ \cline{2-9}

		& \multirow{3}{*}{Shot} & JTT & \multicolumn{1}{c}{$91.3\pm0.02$} & \multicolumn{1}{c}{$90.5\pm0.03$} & \multicolumn{1}{c}{$89.3\pm0.01$} & \multicolumn{1}{c}{$86.5\pm0.02$} & \multicolumn{1}{c}{$83.1\pm0.02$} & \multicolumn{1}{c}{$79.8\pm0.02$} \\
		&  & EIIL & \multicolumn{1}{c}{$90.3\pm0.03$} & \multicolumn{1}{c}{$90.1\pm0.02$} & \multicolumn{1}{c}{$88.3\pm0.01$} & \multicolumn{1}{c}{$86.2\pm0.02$} & \multicolumn{1}{c}{$82.3\pm0.03$} & \multicolumn{1}{c}{$78.5\pm0.02$} \\
		&  & SharpDRO & $\bm{91.6}\pm\bm{0.01}$ & $\bm{90.5}\pm\bm{0.02}$ & $\bm{89.8}\pm\bm{0.02}$ & $\bm{88.7}\pm\bm{0.01}$ & $\bm{86.0}\pm\bm{0.02}$ & $\bm{81.7}\pm\bm{0.01}$ \\  \midrule[0.6pt]

		\multirow{6}{*}{CIFAR100} & \multirow{3}{*}{Snow} & JTT & $67.5\pm0.01$ & $68.1\pm0.02$ & $65.3\pm0.02$ & $64.3\pm0.02$ & $60.2\pm0.02$ & $57.8\pm0.02$ \\
		&  & EIIL & $68.2\pm0.03$ & $69.1\pm0.03$ & $65.2\pm0.02$ & $64.0\pm0.02$ & $61.0\pm0.04$ & $57.5\pm0.04$ \\
		&  & SharpDRO & $\bm{70.6}\pm\bm{0.02}$ & $\bm{69.9}\pm\bm{0.03}$ & $\bm{66.7}\pm\bm{0.03}$ & $\bm{64.4}\pm\bm{0.02}$ & $\bm{61.9}\pm\bm{0.03}$ & $\bm{60.7}\pm\bm{0.03}$ \\ \cline{2-9}

		& \multirow{3}{*}{Shot} & JTT & $66.3\pm0.02$ & $65.3\pm0.03$ & $63.4\pm0.02$ & $56.6\pm0.04$ & $55.5\pm0.04$ & $48.6\pm0.04$ \\
		&  & EIIL & $66.5\pm0.02$ & $65.3\pm0.03$ & $62.8\pm0.04$ & $57.5\pm0.02$ & $56.5\pm0.01$ & $49.5\pm0.01$ \\
		&  & SharpDRO & $\bm{68.9}\pm\bm{0.02}$ & $\bm{66.2}\pm\bm{0.03}$ & $\bm{64.9}\pm\bm{0.03}$ & $\bm{60.1}\pm\bm{0.02}$ & $\bm{58.4}\pm\bm{0.03}$ & $\bm{52.7}\pm\bm{0.02}$ \\  \midrule[0.6pt]

		\multirow{6}{*}{ImageNet30} & \multirow{3}{*}{Snow} & JTT & $86.0\pm0.04$ & $85.8\pm0.02$ & $82.3\pm0.03$ & $80.4\pm0.02$ & $74.6\pm0.02$ & $73.5\pm0.02$ \\
		&  & EIIL & $87.5\pm0.01$ & $85.4\pm0.02$ & $83.5\pm0.04$ & $\bm{81.6}\pm\bm{0.01}$ & $76.3\pm0.01$ & $75.8\pm0.02$ \\
		&  & SharpDRO & $\bm{87.5}\pm\bm{0.03}$ & $\bm{86.7}\pm\bm{0.02}$ & $\bm{85.4}\pm\bm{0.02}$ & $81.5\pm0.03$ & $\bm{78.9}\pm\bm{0.02}$ & $\bm{78.5}\pm\bm{0.03}$ \\ \cline{2-9}

		& \multirow{3}{*}{Shot} & JTT & $86.5\pm0.02$ & $85.4\pm0.03$ & $82.6\pm0.04$ & $79.6\pm0.04$ & $77.2\pm0.04$ & $65.0\pm0.01$ \\
		&  & EIIL & $85.5\pm0.01$ & $86.3\pm0.04$ & $81.6\pm0.02$ & $80.2\pm0.03$ & $75.3\pm0.02$ & $64.4\pm0.03$ \\
		&  & SharpDRO & $\bm{87.3}\pm\bm{0.02}$ & $\bm{87.2}\pm\bm{0.03}$ & $\bm{84.6}\pm\bm{0.03}$ & $\bm{83.2}\pm\bm{0.06}$ & $\bm{79.6}\pm\bm{0.03}$ & $\bm{68.0}\pm\bm{0.03}$ \\  \bottomrule[1pt]
	\end{tabular}
	\vspace{-2mm}
\end{table*}

\section{More Details}
\label{appendix_details}
In this section, we first give a practical implementation of our SharpDRO. Then, we provide more experimental details.

\subsection{Practical Implementation}
\label{appendix_implementation}
Our SharpDRO requires two backward phases, so the time complexity is twice as much as plain training, for efficient sharpness computation, please refer to~\cite{du2022efficient, du2022sharpness, zhang2022ga, zhao2022penalizing, zhao2022ss}. In the first step, we record the label prediction $p$ of each data during inference and simultaneously compute the loss $\mathcal{L}$. Additionally, in the first backward pass, we store the computed gradient $\nabla\mathcal{L}(\theta)$. Further, by adding $\epsilon^*$, we use the perturbed model to compute the second label prediction $\hat{p}$, which is further leveraged to compute the sharpness regularization $\mathcal{R}$. Moreover, in the distribution-agnostic setting, the predictions $p$ and $\hat{p}$ from two forward steps are used to compute the OOD score $\omega_i$. Then, we add the recorded gradient $\nabla\mathcal{L}(\theta)$ back to the model parameter and conduct sharpness minimization over the selected worst-case data. In this way, our SharpDRO can be correctly performed. 

\subsection{Experimental Details}
In our experiments, we choose Wide ResNet-28-2~\cite{zagoruyko2016wide} as our backbone model, using stochastic gradient descent with learning rate $3e-2$ as the base optimizer. The momentum and weight decay factor of the optimizer is set to $0.9$ and $5e-4$, respectively. We run all experiments for 200 epochs with three independent trials and report the average test accuracy with standard deviation.

\section{Additional Experiments}
\label{appendix_exp_results}
In the main paper, we have provided the results using ``Gaussian Noise'' corruption and ``JPEG compression'' corruption, here we conduct additional experiments to show the effectiveness of SharpDRO under ``Snow'' and ``Shot Noise'' corruptions. The results on CIFAR10, CIFAR100, and ImageNet30 datasets in both distribution-aware and distribution-agnostic scenarios are shown in Tables~\ref{tab:appendix_distribution_aware} and~\ref{tab:appendix_distribution_agnostic}. We can see that SharpDRO still performances effectively and surpasses other methods with large margin. Especially, on ImageNet30 dataset in both two problem settings, SharpDRO outperforms second-best method about $3\%$, which indicates the capability of SharpDRO on generalization against different corruptions.

\end{document}